\documentclass[11pt]{article}
\usepackage{geometry}               
\usepackage{amsmath} 
\usepackage{amssymb}  
\usepackage{mathrsfs}
\usepackage{mathtools}
\usepackage{xcolor}
\usepackage{algorithm}
\usepackage{algorithmic}
\usepackage{bm}
\usepackage{multirow}
\usepackage{booktabs}
\usepackage{url}

\usepackage{amsthm}

\newtheorem{theorem}{Theorem}
\newtheorem{proposition}{Proposition}

\newtheorem{problem}{Problem}

\newtheorem{corollary}{Corollary}

\DeclareMathOperator*{\argmin}{arg\,min}

\title{\LARGE \bf
From Prediction Uncertainty to Conformalized Distance Fields for Safe Motion Planning\thanks{This work was supported in part by the Information and Communications Technology Planning and Evaluation (IITP) grants funded by MSIT No. 2022-0-00124, No. 2022-0-00480 and No. RS-2021-II211343, Artificial Intelligence Graduate School Program (Seoul National University),
and the National Research Foundation of Korea (NRF) grant funded by MSIT No. RS-2026-25477173.}
}

\author{
Jaeuk Shin, 
Yoonseok Ra,
and Insoon Yang 
\thanks{The first two authors contributed equally to this work.}
\thanks{
All authors are with the Department of Electrical and Computer Engineering, ASRI,  Seoul National University, Seoul 08826, South Korea, 
        {\tt\small \{sju5379, rys522, insoonyang\}@snu.ac.kr}}%
}

\date{}

\begin{document}

\maketitle
\thispagestyle{empty}
\pagestyle{empty}

 \begin{abstract}
Safe motion planning in dynamic environments requires reasoning about the uncertainty in predicted obstacle motion without sacrificing real-time performance. Existing conformal approaches conformalize a scalar score that aggregates per-obstacle prediction errors, losing spatial coherence and scaling poorly with scene density. We instead conformalize the entire predicted distance field at once. This functional conformal prediction (FCP) framework yields a distribution-free, field-level lower bound, from which safety follows uniformly: any trajectory satisfying the resulting constraint is certified safe, independent of how the control space is sampled. The key enabler is that the residual distance field is empirically low-rank and approximately time-invariant, which makes the bound decomposable in coefficient space. An envelope is fitted offline via functional PCA and a Gaussian-mixture inductive conformal procedure, then refined online by a lightweight adaptive functional conformal (AFCP) update on a low-dimensional vector. This keeps the per-step cost largely insensitive to obstacle count and retains long-run field coverage under distribution shift. We embed the envelope as a tightened safety constraint in a sampling-based model predictive controller, FCP-MPC. On the ETH--UCY pedestrian benchmarks and a dense 3D quadrotor task with up to 280 dynamic obstacles, FCP-MPC attains a favorable balance of safety, feasibility, and efficiency, reaching goals where pointwise and egocentric conformal baselines become too conservative or too expensive, while keeping per-step computation far below online uncertainty-reasoning baselines.
\end{abstract}

\section{Introduction}
Autonomous robots are increasingly asked to operate in environments shared with pedestrians, vehicles, and other moving agents, where safe behavior depends on anticipating how the surrounding world will evolve. The future occupancy of such an environment is unobservable at planning time, so it must be forecast from past observations, and any motion planner that consumes these forecasts inherits their errors. The central difficulty is therefore not prediction or planning in isolation but the coupling of the two: a controller must account for the uncertainty in predicted obstacle motion tightly enough to remain safe, yet cheaply enough to close the loop in real time and in densely populated scenes. Achieving both at once, with quantitative safety guarantees, is the problem we address.

A large body of work embeds prediction uncertainty directly into the planning objective or constraints. Chance-constrained formulations bound the probability of collision~\cite{blackmore2011chance}, and risk-sensitive and distributionally robust planners hedge against tail events or worst-case obstacle distributions~\cite{hakobyan2019risk,fan2021step,hakobyan2021wasserstein}. These methods deliver strong guarantees, but at a price. They typically require committing to an ambiguity set or a parametric/moment-based description of the prediction error, and the associated worst-case optimization is solved online, which grows costly as scene density rises. 
Conformal prediction (CP)~\cite{vovk2005algorithmic,angelopoulos2023conformal,gibbs2021adaptive} offers a complementary, distribution-free route. Rather than constructing such a set or estimating moments, it turns the raw errors of an arbitrary, black-box predictor into sets with finite-sample coverage. CP has been applied to planning among dynamic agents~\cite{lindemann2023safeplanning,dixit2023adaptive}, to perception-based navigation~\cite{mei2025perceive}, and to decision-theoretic control under imperfect predictions~\cite{lekeufack2024decision,contreras2025safe}.

Despite their appeal, most existing CP approaches to motion planning share a common structural limitation. They conformalize a \emph{scalar} score, such as a per-obstacle radius or a per-state margin, and recompute it point-wise online. The online cost then grows with the number of obstacles and the planning horizon, precisely the regime in which real-time safety matters most. Moreover, certifying safety over a \emph{region} rather than at a point requires a worst-case score, which is generally intractable to enforce exactly and is approximated heavily in practice~\cite{mei2025perceive,sundarsingh2026safe}. What is conformalized is thus either cheap but spatially incoherent, or faithful but unaffordable online.

Our point of departure is a representational one. The error of a predicted distance field is not an arbitrary high-dimensional object. Across a scene it is spatially structured and admits an accurate low-dimensional representation. We verify this empirically on the ETH--UCY benchmark, where the residual distance field is approximately time-invariant and low-rank, with a handful of functional principal components capturing most of its variance. This structure suggests treating the prediction error as a \emph{function} over the workspace rather than as a collection of scalars, and conformalizing the entire spatial field at once. Doing so yields a continuous, queryable safety envelope instead of a patchwork of point-wise inflations, and we exploit two consequences throughout. First, the certificate is defined over the entire field, so any trajectory satisfying the constraint is safe regardless of how controls are sampled. Second, because the field is low-rank and approximately time-invariant, the envelope admits an offline--online decomposition that separates the expensive conformalization from real-time planning, keeping the per-step cost largely insensitive to the number of obstacles.

Building on this insight, we introduce a functional conformal prediction (FCP) framework that produces a field-level lower bound on the predicted distance field. We expand the residual field in a data-driven basis obtained by functional principal component analysis, and conformalize the prediction error in the resulting coefficient space with a Gaussian-mixture inductive conformal procedure. This yields a distribution-free envelope available in closed form, as a support function over a union of ellipsoids. The envelope is fitted once, offline, from a large calibration set; online, it is refined by an adaptive functional conformal (AFCP) update that rescales the envelope through a single scalar per horizon step. We embed the conformalized lower bound as a tightened safety constraint in a sampling-based model predictive controller, FCP-MPC, and prove asymptotic closed-loop safety both under exchangeability and, through the online adaptation, under distribution shift.

The contributions of this paper are as follows.
\begin{itemize}
    \item We formulate prediction-uncertainty quantification for safe planning as a functional conformal problem on the residual distance field, yielding a distribution-free, field-level lower bound that is continuous and queryable rather than point-wise. Because the certificate is defined over the whole field, it holds for any trajectory satisfying the constraint, independent of the control sampler.
    
    \item We establish asymptotic closed-loop safety for the hard-constrained MPC whenever the conformalized field satisfies long-run coverage, instantiated with split CP under exchangeability and an adaptive fieldwise update for non-exchangeable streams. We further certify the soft-penalty deployment, which stays safe up to a slack that vanishes as the penalty weight grows and degrades gracefully where the feasible set is empty.
    
    \item To make this field-level bound real-time, we decompose it in coefficient space: a GMM-based inductive conformal envelope is fitted once offline, while online a single low-dimensional adaptive update absorbs distribution shift, keeping per-step computation low and largely decoupled from obstacle count.
    
    \item We demonstrate on the ETH--UCY pedestrian benchmarks and a dense 3D quadrotor task with up to 280 dynamic obstacles that FCP-MPC attains a favorable balance of safety, feasibility, and efficiency. Its per-step cost stays largely insensitive to obstacle count and scales to high densities, where online uncertainty-reasoning baselines become either too conservative to reach the goal or too expensive to meet the real-time budget.\footnote{Our code is available at~\url{https://github.com/CORE-SNU/FCP-MPC}.}
\end{itemize}



\section{Related Works}

\subsection{Navigation in Dynamic Environments: Perception and Prediction}
Safe operation in dynamic environments has classically been decomposed into estimating the evolving occupancy of the world and planning against it. A large body of work tracks moving obstacles by coupling simultaneous localization and mapping with multi-object tracking~\cite{wang2007simultaneous,coue2006bayesian}, or maintains dynamic occupancy representations through particle-based grids~\cite{danescu2011modeling,tanzmeister2014grid} and random-finite-set filters~\cite{nuss2018random}, with continuous particle-based mapping pushed to real-time dense settings~\cite{chen2023continuous}. In parallel, distance fields, whether learned~\cite{park2019deepsdf,ortiz2022isdf} or fused online from depth sensing~\cite{newcombe2011kinectfusion}, have become a favored geometric layer for their smooth, queryable structure. Our formulation inherits this representational viewpoint: rather than reasoning over raw occupancy, we operate directly on the predicted distance field, and we further treat the prediction error of that field as the object to be quantified. On the forecasting side, learned trajectory predictors such as Trajectron++~\cite{salzmann2020trajectron} supply the multi-step obstacle motion that our planner consumes, and their residual errors are precisely what a safety layer must account for.

\subsection{Risk-Aware and Distributionally Robust Motion Planning}
A complementary line of research embeds uncertainty directly into the planning objective or constraints. Chance-constrained formulations bound the probability of collision and have been studied extensively, from convex path planning with obstacles~\cite{blackmore2011chance} to probabilistic collision checking~\cite{du2011probabilistic} and planning in dynamic, uncertain environments~\cite{du2011robot}. Because exact verification of collision constraints against non-convex moving obstacles is generally intractable, these methods typically approximate the free space or the obstacle set, e.g., by unions of convex regions~\cite{zhang2020optimization,deits2015computing}, or accept conservative encodings of the nonconvex feasible set~\cite{da2019collision}. Risk-sensitive planners instead optimize coherent risk measures such as the conditional value-at-risk~\cite{hakobyan2019risk,fan2021step}, while distributionally robust approaches hedge against the worst-case distribution within an ambiguity set~\cite{hakobyan2021wasserstein,hakobyan2023distributionally}. The distributionally robust risk map~\cite{hakobyan2022distributionally}, for instance, certifies CVaR-type safety even when the obstacle motion distribution is only imperfectly learned.

These methods deliver strong guarantees, but the guarantee is only as good as the chosen set or error model, and calibrating it, e.g., selecting a Wasserstein radius or moment bounds, is itself nontrivial. The resulting min--max problem is moreover solved at planning time, so its cost compounds as the obstacle count grows. We share these safety concerns but trade this online worst-case optimization for an offline statistical procedure: a distribution-free conformal envelope is calibrated once, ahead of deployment, leaving only a lightweight coefficient-space update to run online. We do not claim an empirical advantage over distributionally robust planners, and a direct comparison under matched safety levels is left to future work; rather, we position FCP-MPC as a cheaper-online, finite-sample alternative.

\subsection{Conformal Prediction for Safe Planning and Navigation}
Conformal prediction (CP) has recently emerged as a distribution-free route to provable safety, converting heuristic predictor errors into sets with finite-sample coverage. In robotics it has been applied to planning among dynamic agents~\cite{lindemann2023safeplanning,dixit2023adaptive}, to perception-based navigation with statistical assurances~\cite{mei2025perceive}, and to decision-theoretic and ensemble-based control under imperfect predictions~\cite{lekeufack2024decision,contreras2025safe}, with extensions to conformalized semantic maps~\cite{sundarsingh2026safe}. Adaptive conformal prediction~\cite{gibbs2021adaptive} further relaxes the exchangeability requirement, sustaining long-run coverage under distribution shift, and underlies online conformal schemes for motion planning~\cite{dixit2023adaptive}. Most of these approaches, however, conformalize a scalar score online, whose cost grows with obstacle count and horizon. Moreover, the worst-case score definitions they require to certify safety over a region become intractable to enforce exactly~\cite{mei2025perceive,sundarsingh2026safe}.

In contrast, we adopt the functional CP perspective~\cite{lei2015conformal} and treat the residual distance field as a single element of a function space, conformalized as one object rather than point by point. This directly addresses both limitations. First, region-level safety holds by construction: the envelope covers the whole field, so no per-region worst-case score, the quantity prior CP methods cannot enforce exactly, is ever needed. Second, the functional object is low-rank, which functional PCA makes explicit; the few coefficients that describe it are what the controller conformalizes, so the safety check no longer scales with the number of obstacles.

Most closely related is the egocentric conformal prediction (ECP) framework~\cite{shin2025egocentric}, our ECP-MPC baseline. ECP shares our premise of scoring how much closer obstacles are than predicted, rather than inflating a scalar margin for every obstacle, but it attaches that score to each robot \emph{state} and conformalizes it online. Its uncertainty reasoning therefore runs in the planning loop and scales with the number of states queried, degrading as obstacle density grows (Section~\ref{sec:res}). We instead conformalize the residual field as one functional object computed offline, so the certificate holds for any sampled trajectory (Theorems~\ref{thm:asymptotic_safety}--\ref{thm:asymptotic_safety_afcp}) and the per-step cost is essentially independent of how many obstacles or states the scene contains.

\section{Preliminaries}\label{sec:preliminaries}

We use the following notation throughout the paper: For $f, g \in L^2(\mathcal{X}, d\mathrm{x})$, we denote their inner product by $\langle f, g \rangle = \int_{\mathcal{X}} f(\mathrm{x})g(\mathrm{x}) d\mathrm{x}$. For $\mathcal{O} \subseteq \mathbb{R}^d$ and $\mathrm{x}\in \mathbb{R}^d$, $D(\mathrm{x}, \mathcal{O})$ denotes the closest distance from $\mathrm{x}$ to $\mathcal{O}$: $D(\mathrm{x}, \mathcal{O})\coloneqq \inf_{\mathrm{x}'\in \mathcal{O}} \| \mathrm{x} - \mathrm{x}'\|$.

Our machinery builds on conformal prediction (CP), a distribution-free
framework that turns a heuristic uncertainty score into sets with finite-sample
coverage. We recall  two building blocks: split conformal prediction, which we use offline, and its adaptive variant, which underlies our online update. The functional extension specific to distance fields is developed in Section~\ref{sec:functional_projection}.

\subsection{Split Conformal Prediction}
Given exchangeable samples $\{z_n\}_{n=1}^N$ and a nonconformity score
$s(\cdot)\in\mathbb{R}$ (larger meaning a worse fit), fix a target miscoverage
level $\alpha\in(0,1)$ and set
\begin{equation*}
\hat q = \mathsf{Quantile}_{1-\alpha}\bigl(\{s(z_n)\}_{n=1}^N \cup \{+\infty\}\bigr),
\end{equation*}
where the $+\infty$ term supplies the finite-sample correction. 
For a fresh exchangeable sample $z_{N+1}$,
\begin{equation}
\mathbb{P}\bigl[s(z_{N+1})\le \hat q\bigr]\ge 1-\alpha,
\label{eq:split_cp}
\end{equation}
with no assumption on the data distribution beyond exchangeability~\cite{lei2015conformal}.

\subsection{Adaptive Conformal Prediction}
When the stream is non-exchangeable (e.g.\ under distribution shifts), the marginal
guarantee~\eqref{eq:split_cp} may fail. Adaptive conformal prediction
(ACP)~\cite{gibbs2021adaptive} instead enforces a
\emph{long-run} guarantee by updating the threshold online from the observed
miscoverage: with step size $\gamma>0$ and
$\mathrm{err}_t=\mathbb{I}[\,s_t> \hat q_t\,]$,
\begin{equation}
\hat q_{t+1}=\hat q_t+\gamma\,(\mathrm{err}_t-\alpha),
\label{eq:aci}
\end{equation}
which drives the empirical miscoverage $\tfrac1T\sum_{t}\mathrm{err}_t$ to $\alpha$
for an arbitrary, even adversarial, sequence.

\section{Problem Formulation}\label{sec:prob}

Let $\mathcal{Q}$ be the configuration space of a robot, and  $\mathrm{x} = \mathrm{x}(\mathrm{q}) \in \mathcal{X} \subseteq \mathbb{R}^d$ the position in the (compact) world $\mathcal{X}$  corresponding to $\mathrm{q}\in \mathcal{Q}$ through the forward kinematics, with $d = 2$ or $3$.
We write $\mathcal{A} = \mathcal{A}(\mathrm{q}) \subseteq \mathcal{X}$ for the robot's footprint, a closed set, at configuration $\mathrm{q}$, and $\mathcal{A}_t \coloneqq \mathcal{A}(\mathrm{q}_t)$.
With control input $\mathrm{u} \in \mathcal{U}$, the robot evolves as \[
\mathrm{q}_{t+1} = f(\mathrm{q}_t, \mathrm{u}_t),\quad t = 0, 1, \ldots.
\]
The environment is dynamic. Over a time window $t = 0, \ldots, T$, let $\mathcal{O}_t \subseteq \mathcal{X}$ denote the union of moving obstacles at time $t$, typically realized as an occupancy map and \emph{completely unknown} at planning time.
We model $\{\mathcal{O}_t\}_{t=0}^T$ as a set-valued random process whose randomness captures the environmental variation, and assume each $\mathcal{O}_t$ is independent of $\mathrm{q}_0, \mathrm{u}_0, \ldots, \mathrm{q}_t$, so that the robot's motion does not influence the occupancy.\footnote{Extending this to bidirectional robot--agent interaction is left as future work.}

We are primarily interested in the following problem:
\begin{problem}[Motion Planning with Probabilistic Safety Guarantee]\label{problem:general-motion-planning}

Design a policy $\mathrm{u}_t = \pi(\mathrm{q}_t, \mathcal{H}_t)$ that solves
\begin{align}\label{eq:path-collision-constraint-general}
\min_{\pi}&\quad  \sum_{t=0}^{T-1} \ell(\mathrm{q}_t, \mathrm{u}_t) + \ell(\mathrm{q}_T) \nonumber \\
\text{s.t.}&\quad \mathrm{q}_{t+1} = f(\mathrm{q}_t, \mathrm{u}_t), \quad 0 \leq t < T, \nonumber \\
&\quad \mathbb{P}\left[ \mathcal{A}_t\cap \mathcal{O}_t = \varnothing \right] \geq  1 - \alpha, \quad 1\leq t \leq T.
\end{align}
\end{problem}

Here safety is required at the \emph{per-step} level. That is, each step is collision-free with probability at least $1-\alpha$, which is the guarantee our conformal construction certifies in the closed loop (Theorems~\ref{thm:asymptotic_safety}--\ref{thm:asymptotic_safety_afcp}).
Throughout, we set the conformal miscoverage level to this same $\alpha$, so that a field-coverage guarantee at level $1 - \alpha$ translates directly into per-step safety at level $1 - \alpha$.

The condition $\mathcal{A}_t \cap \mathcal{O}_t = \varnothing$ may be strengthened into
\begin{equation}\label{eq:collision-constraint-general}
\inf_{\mathrm{x} \in \mathcal{A}_t} D(\mathrm{x}, \mathcal{O}_t) \geq \delta
\end{equation}
where $\delta > 0$ is a user-specified margin between $\mathcal{A}_t$ and $\mathcal{O}_t$.
Since $\mathcal{O}_t$ is generally non-convex, exact verification of this condition is intractable even when $\mathcal{A}_t$ is a convex polytope. A typical remedy  is to approximate $\mathcal{O}_t$ or the free space $\mathcal{X} \setminus \mathcal{O}_t$ by a union of convex regions~\cite{zhang2020optimization,deits2015computing}. Instead of inheriting this assumption, we  assume the footprint $\mathcal{A}(\mathrm{q})$ is a ball of radius $r_{\text{robot}}$ centered at $\mathrm{x}(\mathrm{q})$:
\begin{equation*}
\mathcal{A}(\mathrm{q}) = \{\mathrm{x}' \in \mathcal{X}: \left\Vert \mathrm{x}' - \mathrm{x}(\mathrm{q}) \right\Vert \leq r_{\text{robot}}  \}.
\end{equation*}
Then 
\[
\inf_{\mathrm{x} \in \mathcal{A}_t} D(\mathrm{x}, \mathcal{O}_t) = D(\mathrm{x}_t, \mathcal{O}_t) - r_{\text{robot}},
\]
and~\eqref{eq:collision-constraint-general} reduces to
\begin{equation}\label{eq:collision-constraint}
D(\mathrm{x}_t, \mathcal{O}_t) \geq \delta + r_{\text{robot}},
\end{equation}
a single query of $D(\cdot, \mathcal{O}_t)$ at the center $\mathrm{x}_t$.
 The condition~\eqref{eq:path-collision-constraint-general} is now translated into
\begin{equation}\label{eq:path-collision-constraint}
\mathbb{P}\left[ D(\mathrm{x}_t, \mathcal{O}_t) \geq \delta + r_{\text{robot}} \right] \geq 1 - \alpha, \quad 1\leq t \leq T.
\end{equation}
To solve Problem~\ref{problem:general-motion-planning}, the policy $\pi$ should incorporate the information  collected up to $t$, $\mathcal{H}_t \coloneqq(\mathcal{O}_0, \ldots, \mathcal{O}_t)$, to secure future safety while maintaining  deployment efficiency.
A feasible approach is to define $\pi(\mathrm{q}_t, \mathcal{H}_t)$ as the solution of the following MPC problem:
\begin{align}
\min_{\mathbf{q}, \mathbf{u}}&\ \mathfrak{L}(\mathbf{q}, \mathbf{u}, \mathbf{x}^{\text{ref}}) \nonumber \\
\text{s.t.}
&\ \ \mathrm{q}_{t|t} = \mathrm{q}_t, \nonumber \\
&\ \ \mathrm{q}_{t+i+1|t} = f(\mathrm{q}_{t+i|t}, \mathrm{u}_{t+i|t}), \quad 0 \leq i < N, \nonumber\\
&\ \ D(\mathrm{x}_{t+i|t}, \mathcal{O}_{t+i|t}) \geq \delta + r_{\text{robot}}, \quad 1\leq i \leq N, \label{eq:mpc-constraint-naive}
\end{align}
where $(\mathrm{x}_{t+1|t}, \ldots,\mathrm{x}_{t+N|t})$ is a  path planned  at $t$ from $\mathrm{x}_t$ and $\mathcal{H}_t$, and  $(\mathcal{O}_{t+1|t}, \ldots, \mathcal{O}_{t+N|t})$ are \emph{predicted} occupancies at $t+1, \ldots, t+N$ from $\mathcal{H}_t$.\footnote{While $\mathcal{O}_t$ is partially observable in practice, we assume  full observability throughout. Under partial observability, one may instead use a \emph{conservative} estimate $\hat{\mathcal{O}}_t \supseteq \mathcal{O}_t$ that treats occluded and unknown space as occupied, and choose $\mathbf{x}^{\text{ref}}$ to balance exploration and efficiency via a \emph{frontier} method~\cite{yamauchi1998frontier}. Since this only inflates the distance-field estimate, all theoretical claims remain valid at the cost of added conservativeness. A full treatment of partial observability is left as future work.} The objective $\mathfrak{L}$ is defined as
\begin{equation*}
\mathfrak{L}(\mathbf{q}, \mathbf{u}, \mathbf{x}^{\text{ref}}) = \sum_{i=0}^{N-1} \ell(\mathrm{q}_{t+i|t}, \mathrm{u}_{t+i|t}, \mathrm{x}^{\text{ref}}_{t+i|t}) + \ell(\mathrm{q}_{t+N|t}, \mathrm{x}^{\text{ref}}_{t+N|t}).
\end{equation*}

Several approaches are available for computing the predicted occupancies $\mathcal{O}_{t+i|t}$.
One common and mature route is to estimate the kinematic states of dynamic agents by
detection and tracking of the dynamic agents from $\mathcal{H}_t$, then apply motion prediction models~\cite{salzmann2020trajectron,yuan2021agentformer,zhou2022hivt,fu2025moflow,lee2025koopcast} to predict their future motions, which is back-projected into $\mathcal{X}$ to infer $(\mathcal{O}_{t+1|t}, \ldots, \mathcal{O}_{t+N|t})$. An alternative is end-to-end prediction, deploying a model that consumes raw lidar scans and predicts future shapes directly. We adhere to the former.

Since the predictions are imperfect, the associated distance functions may deviate from the true occupancies of the environment. Abbreviating $D_{t+i|t}(\mathrm{x}) \coloneqq D(\mathrm{x}, \mathcal{O}_{t+i|t})$, we define the \emph{score function} as
\begin{equation}
S_{t+i|t} (\mathrm{x}) \coloneqq D_{t+i|t}(\mathrm{x}) - D_{t+i}(\mathrm{x}), \quad \mathrm{x}\in \mathcal{X}.
\end{equation}
When $S_{t+i|t}(\mathrm{x}) > 0$, the predicted distance exceeds the true one: the obstacle is in fact closer than predicted, so the prediction \emph{overestimates} safety at $\mathrm{x}$. Correcting this underestimation of risk is exactly what is needed to enforce~\eqref{eq:path-collision-constraint}, and is the target of our conformal construction.

\noindent
{\bf Remark on what is proven.} The constraint~\eqref{eq:path-collision-constraint} (equivalently the per-step chance constraint in~\eqref{eq:path-collision-constraint-general}) states the desired \emph{per-step} safety semantics, i.e., a bound on the probability of collision at each individual step. Our formal guarantees do \emph{not} certify this per-step probability pointwise in time. Rather, they establish its \emph{long-run empirical} closed-loop counterpart for the receding-horizon controller: the time-averaged frequency of collision-free steps is at least $1-\alpha$ with probability one (Theorems~\ref{thm:asymptotic_safety}--\ref{thm:asymptotic_safety_afcp}).

\section{Functional Conformal Prediction for Provable Distance Function Coverage}
\label{sec:functional_projection}

\subsection{Why Functional Conformal Prediction?}\label{subsec:fcp-why}
In contrast to prior conformal-prediction approaches in robotics, where a single scalar score function is typically considered, the error signal $S_{t+i|t}$ here is \emph{functional}. Each  $S_{t+i\mid t}\colon\mathcal{X}\to\mathbb{R}$ is an element of the Hilbert space $\mathscr{H} := L^2(\mathcal{X})$, not a scalar. 
This motivates a CP framework tailored to such functional objects, namely
functional conformal prediction (FCP)~\cite{lei2015conformal}. Conformalizing
the field as a whole is what yields a certificate that holds for any trajectory
satisfying it, independent of the sampler (Section~\ref{sec:safety}). Its
\emph{practicality}, in turn, rests on two empirical properties of the residual
field, which we verify on the ETH--UCY benchmark~\cite{lerner2007crowds,pellegrini2009you}.

First, the residual field is approximately \emph{time-invariant}: per-cell residual statistics from two temporally disjoint halves of each scene agree closely, so $S_{t+i\mid t}\stackrel{d}{\approx}S_{t'+i\mid t'}$ for $t\neq t'$. Second, the field is \emph{low-rank}: a few functional principal components already capture most of its variance, so it admits the compact expansion:
\begin{equation}
S_{t+i\mid t}(\mathrm{x}) \;\approx\; \bar{S}_i(\mathrm{x}) + \sum_{k=1}^{r}\xi^{(k)}_{t}\,\psi_k(\mathrm{x}), \qquad r\ll\infty,
\label{eq:lowrank-score}
\end{equation}
with FPCA modes $\{\psi_k\}_{k=1}^{r}\subset L^2(\mathcal{X})$ that are stable across $t$. 
The uncertainty is thus spatially structured and compressible, even though it is not reducible to a simple geometric cue such as path curvature or visitation density. 
These two properties, which we quantify on the ETH--UCY benchmark in Section~\ref{sec:res-properties} (Figs.~\ref{fig:fcp_stationarity} and~\ref{fig:fcp_lowrank}), are exactly what make the field-level envelope cheap to deploy. The modes
$\{\psi_k\}$ and the resulting envelope $\mathsf{U}_{t+i|t}$ are estimated once, offline,
while online the controller merely evaluates $\mathsf{U}_{t+i|t}(\mathrm{x})$ point-wise along its
rollouts, keeping the per-step cost low. To be self-contained, we now explain how FCP turns these prediction errors into a safety margin.

\subsection{Sketch of Core Pipeline}
\label{sec:score-comp}

Our goal is to construct a high-probability upper bound $\mathsf{U}_{t+i\mid t}$ for $S_{t+i|t}$:
\begin{equation}\label{eq:coverage-guarantee}
\mathbb{P}\!\left[
S_{t+i\mid t}(\mathrm{x}) \le \mathsf{U}_{t+i\mid t}(\mathrm{x}) \;\forall \mathrm{x}\in\mathcal{X}
\right] \ge 1-\alpha.
\end{equation}
Once such $\mathsf{U}_{t+i|t}$ is identified, it yields a probabilistic lower bound of $D_{t+i}$:
\begin{equation}\label{eq:lower-bound-coverage}
\mathbb{P}\bigl[ D_{t+i}(\mathrm{x}) \geq \underbrace{D_{t+i|t}(\mathrm{x}) - \mathsf{U}_{t+i|t}(\mathrm{x})}_{\coloneqq \mathsf{L}_{t+i|t}(\mathrm{x})} \;\forall  \mathrm{x}\in \mathcal{X}   \bigr] \geq 1 - \alpha.
\end{equation}
We then replace the constraint~\eqref{eq:mpc-constraint-naive} by 
\begin{equation}\label{eq:mpc-constraint-corrected}
\mathsf{L}_{t+i|t}(\mathrm{x}_{t+i|t}) \geq  r_{\text{robot}} + \delta, \quad 1\leq i \leq N.
\end{equation}
However, directly computing such a bound in the infinite-dimensional Hilbert
space $\mathscr{H}$ is generally intractable and often unnecessary.
Instead, we leverage the functional conformal prediction framework
\cite{lei2015conformal} and  bound  a finite-dimensional projection
of $S_{t+i\mid t}$.

We adopt a \emph{divide-and-conquer} strategy. Decompose
\begin{equation}\label{eq:score-decomposition}
S_{t+i|t} = \breve{S}_{t+i\mid t} + \underbrace{(S_{t+i|t} - \breve{S}_{t+i|t})}_{\coloneqq R_{t+i|t}},
\end{equation}
where
$\breve{S}_{t+i|t}$ is a \emph{finite-dimensional} projection of $S_{t+i|t}$, constructed in Section~\ref{sec:fcp_ub}. 
We bound $\breve{S}_{t+i\mid t}$ from above by $\breve{\mathsf{U}}_{t+i|t}$ via FCP, and bound the projection residual $R_{t+i|t}$, which is negligible compared to $\breve{S}_{t+i\mid t}$, uniformly by standard CP: 
\[
\Vert R_{t+i|t}\Vert_{\infty} \leq \varepsilon_i.
\]
Combining the two gives the desired bound on $S_{t+i|t}$:
\begin{equation}\label{eq:final-envelope}
\mathsf{U}_{t+i|t}(\mathrm{x}) \coloneqq \breve{\mathsf{U}}_{t+i|t}(\mathrm{x}) + \varepsilon_i.
\end{equation}

\subsection{FCP-Based Upper Bound}\label{sec:fcp_ub}

The most intricate part of the pipeline is constructing  the finite-dimensional projection $\breve{S}_{t+i|t}$ of $S_{t+i|t}$ and its high-confidence upper bound $\breve{\mathsf{U}}_{t+i|t}$. 
We define a $p_i$-dimensional subspace $\mathscr{V}_i \subset \mathscr{H}$ and let 
\[
\breve{S}_{t+i|t} \coloneqq \Pi_{\mathscr{V}_i}(S_{t+i|t}),
\]
where $\Pi_{\mathscr{V}_i} : \mathscr{H} \to \mathscr{V}_i$ denotes the projection operator.
It then remains to find a confidence set $\mathscr{T}_{t+i|t}$  satisfying a distribution-free coverage guarantee of the form
\begin{equation}
\mathbb{P}\left[
\breve{S}_{t+i|t} \in \mathscr{T}_{t+i\mid t}
\right] \ge 1 - \alpha/2,
\label{eq:functional_projected_coverage}
\end{equation}
where $\mathscr{T}_{t+i\mid t} \subset \mathscr{V}_i$ is a conformal prediction set
 on the projected space.
Below we describe a specific choice of $\mathscr{V}_i$ and how $\breve{\mathsf{U}}_{t+i|t}$ is derived from  $\mathscr{T}_{t+i|t}$ following~\cite{lei2015conformal}.

\subsubsection{Data-driven basis functions}
In our setting, the projection subspace $\mathscr{V}_i\subset\mathscr{H}$
is learned from data at each horizon index $i=1, \ldots, N$ via functional principal components.
Let $\{S^{(n)}_{t+i\mid t}\}_{n=1}^N$ denote a collection of score functions
at step $i$.
We choose an orthonormal basis
$\{\psi_{i,j}\}_{j=1}^{p_i}$ spanning $\mathscr{V}_i$ as the $p_i$ leading principal components
of this  sample, in the standard $L^2$ sense, and collect the basis evaluations at $\mathrm{x}$ into the vector
\begin{equation*}
\psi_i(\mathrm{x}) \coloneqq \bigl(\psi_{i,1}(\mathrm{x}),\ldots,\psi_{i,p_i}(\mathrm{x})\bigr)^\top \in \mathbb{R}^{p_i}.
\end{equation*}
For any $S\in\mathscr{H}$, its coefficient vector in this basis is
\begin{equation}\label{eq:coefficient}
\xi_{i}(S)
\;\coloneqq\;
\bigl(\langle S,\psi_{i,1}\rangle,\ldots,\langle S,\psi_{i,p_i}\rangle\bigr)^\top
\in\mathbb{R}^{p_i},
\end{equation}
so that the projection admits the expansion
\begin{equation*}
(\Pi_{\mathscr{V}_i} S)(\mathrm{x}) \;=\; \sum_{j=1}^{p_i} \xi_{i,j}(S)\,\psi_{i,j}(\mathrm{x}).
\end{equation*}
The dimension $p_i$ of $\mathscr{V}_i$ may be fixed or selected from the sample size and a target approximation accuracy;  we treat it as a modeling parameter.

\paragraph{Inductive conformal prediction in coefficient space}
\label{sec:icp_coeff}

We adopt the inductive conformal predictor (ICP) framework~\cite{lei2015conformal} to conformalize a high-density region in coefficient space.
We split the coefficient samples $\{\xi^{(n)}\}_{n=1}^N$
into a training subset and a calibration subset $\textsf{Cal}$.
A parametric conformity score $g(\xi)$ is fitted on the training subset, and a threshold is chosen using the empirical quantile on the calibration subset.

\paragraph{Gaussian mixture model and ellipsoidal coefficient regions}
\label{sec:gmm_ellipsoids}

We model the distribution of the projected coefficients $\xi \in \mathbb{R}^{p_i}$
using a $K$-component Gaussian mixture model (GMM).
Specifically, we fit the mixture weights $\widehat{\pi}_k$, means
$\widehat{\mu}_k$, and covariance matrices $\widehat{\Sigma}_k$
to the training coefficients and obtain the estimated density
\begin{equation*}
\widehat f(\xi)
\;=\;
\sum_{k=1}^K \widehat{\pi}_k\,
\mathcal{N}\!\left(\xi;\widehat{\mu}_k,\widehat{\Sigma}_k\right),
\end{equation*}
where $\mathcal{N}(\cdot;\mu,\Sigma)$ denotes the multivariate Gaussian density (not to be confused with the FPCA basis functions $\psi_{i,j}$).

\paragraph{Conformity score}
Following the construction in Eqs.~(8)--(9) of \cite{lei2015conformal},
we define a max-component pseudo-density conformity score
from the fitted GMM parameters:
\begin{equation*}
g(\xi)
\;\coloneqq\;
\max_{1 \le k \le K}
\widehat{\pi}_k\,
\mathcal{N}\!\left(\xi;\widehat{\mu}_k,\widehat{\Sigma}_k\right).
\end{equation*}
This score assigns higher conformity to coefficients lying in high-density
regions of at least one mixture component, thereby yielding
ellipsoidal regions in the coefficient space.

Let $\{\xi^{(n)}\}_{n\in\mathsf{Cal}}$ be the coefficients of the calibration set and define
\begin{equation*}
\lambda_i \coloneqq \mathsf{Quantile}_{\alpha/2}\Bigl(\{ g(\xi^{(n)}) \}_{n\in\mathsf{Cal}} \cup \{-\infty\} \Bigr).
\end{equation*}
Then the ICP prediction set in coefficient space is
\begin{equation}\label{eq:projection-coverage}
\begin{gathered}
\mathscr{T}_{t+i\mid t} \coloneqq \{\xi\in\mathbb{R}^{p_i}: g(\xi)\ge \lambda_i\}, \\
\mathbb{P}\bigl[\xi_{t+i\mid t}\in\mathscr{T}_{t+i\mid t}\bigr]\ge 1-\alpha/2,
\end{gathered}
\end{equation}
where $\xi_{t+i\mid t}\coloneqq\xi_i(S_{t+i\mid t})$.
We derive a high-confidence upper bound of $\breve{S}_{t+i|t}$ as
\begin{equation*}
\breve{\mathsf{U}}_{t+i|t}(\mathrm{x}) \coloneqq \sup_{\xi\in\mathscr{T}_{t+i\mid t}} \xi^\top \psi_i(\mathrm{x}), \quad \mathrm{x} \in \mathcal{X}.
\end{equation*}
Since $\breve{S}_{t+i|t}(\mathrm{x})=\xi_{t+i\mid t}^\top \psi_i(\mathrm{x})$, the coverage~\eqref{eq:projection-coverage} immediately gives
\begin{equation}\label{eq:fcp-prob-bound}
\mathbb{P}\left[\breve{S}_{t+i|t}(\mathrm{x})\leq \breve{\mathsf{U}}_{t+i|t}(\mathrm{x})\; \text{for all}\; \mathrm{x} \in \mathcal{X}\right]\ge 1-\frac{\alpha}{2}.
\end{equation}
It turns out that $\breve{\mathsf{U}}_{t+i|t}$ admits a closed-form expression
\begin{equation}\label{eq:upper-envelope-closed}
\breve{\mathsf{U}}_{t+i|t}(\mathrm{x})= \max_{1\leq k\leq K} \left\{ \widehat\mu_k^\top \psi_i(\mathrm{x})
+
r_{k,i}\sqrt{\psi_i(\mathrm{x})^\top \widehat\Sigma_k\, \psi_i(\mathrm{x})} \right\},
\end{equation}
whose detailed derivation is given in Appendix~\ref{app:derivation}.

\subsection{Bounding Projection Residuals}
To translate coefficient-space uncertainty back to the original discretized function,
we account for the projection residual in sup norm.
Define the reconstruction error for sample $y^{(n)}$ by
\[
e^{(n)} \;\coloneqq\; \big\|y^{(n)} - \Pi_{\mathscr{V}_i}  \big( y^{(n)} \big) \big\|_{\infty}.
\]
Using the calibration split, we set the projection residual to the conformal quantile
\begin{equation}
\varepsilon_i \;\coloneqq\; \mathsf{Quantile}_{1-\alpha/2}\bigl( \{e^{(n)}\}_{n\in\mathrm{cal}} \cup \{+\infty\} \bigr),
\label{eq:epsilon_quantile}
\end{equation}
where the $+\infty$ term supplies the finite-sample correction, 
which yields
\begin{equation}\label{eq:projection-res-prob-bound}
\mathbb{P}\left[ |R_{t+i|t}(\mathrm{x})| \leq \varepsilon_i \; \text{for all}\; \mathrm{x}\in \mathcal{X} \right] \geq 1 - \alpha/2.
\end{equation}
By a union bound over the two events~\eqref{eq:fcp-prob-bound} and~\eqref{eq:projection-res-prob-bound}, each holding with probability at least $1-\alpha/2$, the decomposition~\eqref{eq:score-decomposition} and the envelope~\eqref{eq:final-envelope} yield the desired coverage guarantee~\eqref{eq:coverage-guarantee} at level $1-\alpha$.

\noindent
{\bf Remark on practical implementation.}
The projection coefficients $\xi_{i,j}(S)=\langle S,\psi_{i,j}\rangle$ generally require
integration over $\mathcal X$, which is expensive when $S$ is only available through costly
distance-transform evaluations on a fine grid.
Therefore, in implementation we approximate the projection (and its point-wise evaluations)
using only finitely many grid points $\bar{\mathcal X}\subset\mathcal X$. For later use, we define the \emph{discretization resolution} of $\bar{\mathcal{X}}$ as
\begin{equation}\label{eq:resolution}
\delta_d \coloneqq \sup_{\mathrm{x} \in \mathcal{X}} \min_{\bar{\mathrm{x}} \in \bar{\mathcal{X}}} \|\mathrm{x} - \bar{\mathrm{x}} \|,
\end{equation}
which is finite since $\mathcal{X}$ is bounded.

\section{FCP-Informed Safe Motion Planning}
\label{sec:cp_mpc}
We now embed the conformalized distance lower bound into a receding-horizon motion planner, following the  safe-planning template of~\cite{lindemann2023safeplanning} but adapting it to our \emph{fieldwise} representation.
The conformalization is performed over the entire field offline and cached in compact coefficient form; online, the controller neither re-fits the GMM nor recomputes per-point distances, but consults the cached field and applies a lightweight scalar update that also absorbs any deployment-time distribution shift (Section~\ref{subsec:online_update}). Beyond keeping the per-step cost low, this hands the planner a continuous, queryable safety field, directly suited to scoring the many candidate rollouts of an MPC step. Figure~\ref{fig:overview} summarizes the pipeline.

\begin{figure*}[t]
  \centering
  \includegraphics[width=0.9\linewidth]{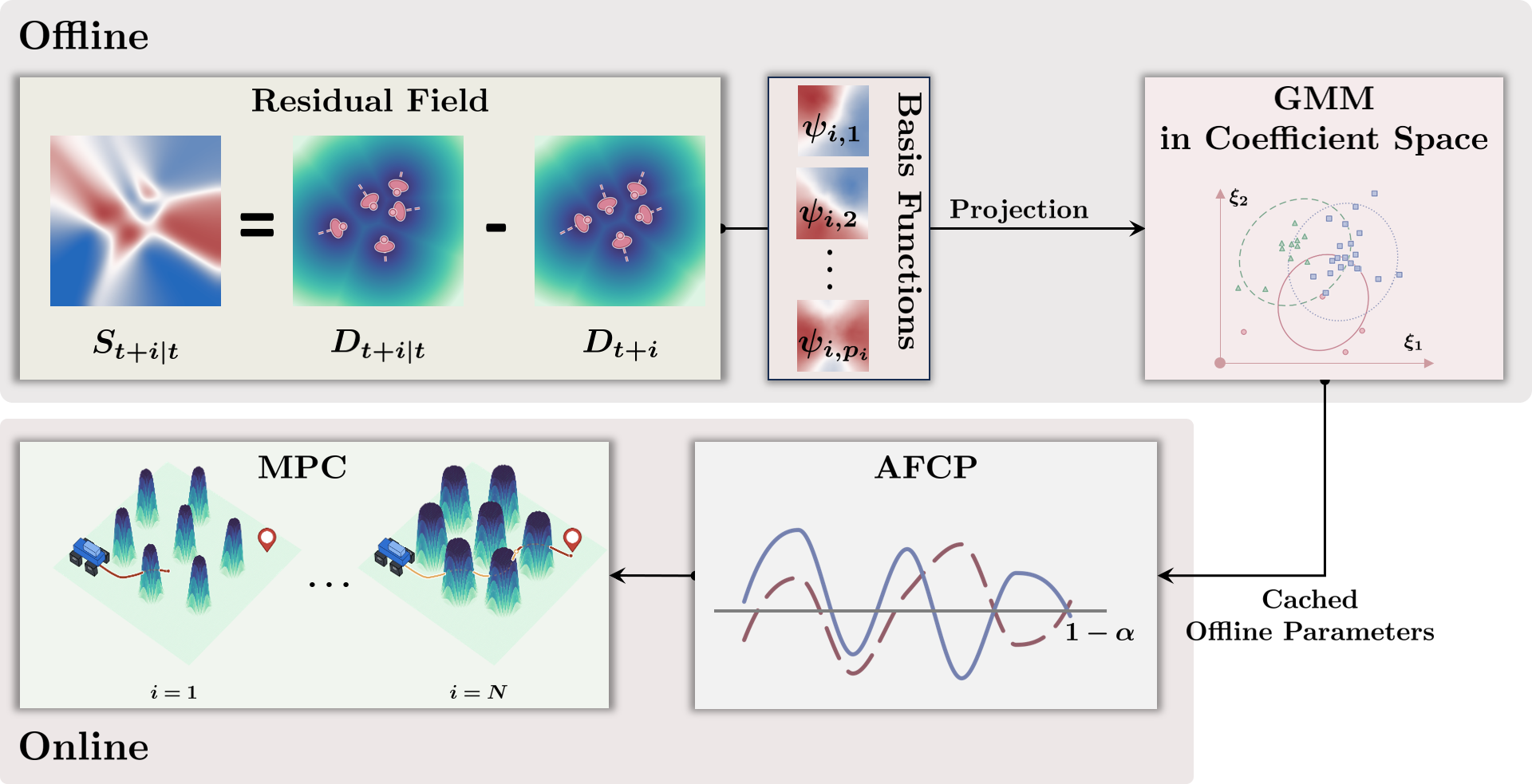}
  \caption{In the offline stage, snapshots of residual distance fields are collected, and a set of basis functions is obtained via functional PCA. This induces a projection of each residual field onto a low-dimensional coefficient space, in which the distribution of the projected points is approximated by a GMM and FCP is performed. In the online stage, AFCP-based correction is applied to handle test-time distribution shifts, and the resulting conformalized distance field is incorporated into the MPC problem.}
  \label{fig:overview}
\end{figure*}

\subsection{MPC Formulation}
\label{subsec:mpc_formulation}
Incorporating the conformalized constraint~\eqref{eq:mpc-constraint-corrected} into~\eqref{eq:mpc-constraint-naive} leads to the following MPC formulation:
\begin{align}
\min_{\mathbf{q}, \mathbf{u}}&\ \ \mathfrak{L}(\mathbf{q}, \mathbf{u}, \mathbf{x}^{\text{ref}}) \nonumber \\
\text{s.t.}
&\ \ \mathrm{q}_{t|t} = \mathrm{q}_t, \nonumber \\
&\ \ \mathrm{q}_{t+i+1|t} = f(\mathrm{q}_{t+i|t}, \mathrm{u}_{t+i|t}), \quad 0 \leq i < N, \nonumber\\
&\ \ \mathsf{L}_{t+i|t}(\bar{\mathrm{x}}_{t+i|t}) \geq r_{\text{robot}} + \delta + \delta_{d}, \quad 1\leq i \leq N,\label{eq:mpc-constraint-fcp}
\end{align}
where $\delta_d > 0$ is the discretization resolution of $\mathcal{X}$ defined in~\eqref{eq:resolution}, and $\bar{\mathrm{x}}_{t+i|t}$ denotes the \emph{projection} of $\mathrm{x}_{t+i|t} = \mathrm{x}(\mathrm{q}_{t+i|t})$ onto the finite set $\bar{\mathcal{X}}$:
\begin{equation*}
\bar{\mathrm{x}}_{t+i|t} \coloneqq \argmin_{\bar{\mathrm{x}} \in \bar{\mathcal{X}}} \| \bar{\mathrm{x}} - \mathrm{x}_{t+i|t} \|. 
\end{equation*}
To solve this, a sampling-based MPC, such as model predictive path integral (MPPI)~\cite{williams2017information}, can be used as an affordable implementation.\footnote{In implementation we enforce a horizon-dependent margin
$\rho_i = r_{\text{robot}}+\delta+\delta_d-\Delta_i$ with
$\Delta_i=\tfrac12\,a_{\mathrm{lat}}\bigl((i-1)\Delta t\bigr)^2$
($a_{\mathrm{lat}}=v_{\max}\omega_{\max}$ in 2D), the lateral distance the robot can
deviate by step $i$ (cf.~\cite{fraichard2004inevitable}). Subtracting it relaxes the
margin at far horizons, where the robot can still steer clear. Since $\Delta_1=0$,
the applied step keeps full clearance, so the guarantees of Section~\ref{sec:safety}
are unaffected.}

\noindent
{\bf Hard constraints vs. soft penalization.}
While theoretically favorable, the conformalized constraints~\eqref{eq:mpc-constraint-fcp} may still produce extremely small feasible sets when $i$ is large. Moreover, the fieldwise conformalization is more than a binary feasibility test: at every point it furnishes a signed, continuous margin.
This motivates the notion of \emph{soft constraints}~\cite{scokaert1999feasibility,kerrigan2000soft}. Instead of enforcing the hard constraints explicitly, we fold them into the MPC objective as a graded term that penalizes violations. 
Specifically, we define
\begin{equation}
g_{t+i|t}(\mathrm{x}):=  r_{\text{robot}} +\delta + \delta_{d} - \mathsf{L}_{t+i|t}(\mathrm{x}),
\label{eq:soft_margin}
\end{equation}
so that the constraint~\eqref{eq:mpc-constraint-fcp} is identical to $g_{t+i|t}(\mathrm{x})\le 0$.  
The soft version of the MPC is then 
\begin{align}\label{eq:soft-mpc}
\min_{\mathbf{q}, \mathbf{u}}&\ \mathfrak{L}_{w}(\mathbf{q}, \mathbf{u}, \mathbf{x}^{\text{ref}}) \\
\text{s.t.}
&\ \ \mathrm{q}_{t|t} = \mathrm{q}_t, \nonumber \\
&\ \ \mathrm{q}_{t+i+1|t} = f(\mathrm{q}_{t+i|t}, \mathrm{u}_{t+i|t}), \quad 0 \leq i < N, \nonumber
\end{align}
where $\mathfrak{L}_{w}$ is defined as
\begin{equation*}
\mathfrak{L}_{w}(\mathbf{q}, \mathbf{u}, \mathbf{x}^{\text{ref}}) = \mathfrak{L}(\mathbf{q}, \mathbf{u}, \mathbf{x}^{\text{ref}}) + w \sum_{i=1}^N \left[ g_{t+i|t}(\bar{\mathrm{x}}_{t+i|t}) \right]^2_+
\end{equation*}
for a user-defined constant $w > 0$. The penalty term discourages incursions into the conformalized unsafe region ($g_{t+i|t}>0$). The weight $w>0$ is fixed rather than adapted, so~\eqref{eq:soft-mpc} is a quadratic-penalty relaxation~\cite{nocedal2006numerical} of its hard counterpart~\eqref{eq:mpc-constraint-fcp}: its minimizer approaches the hard-constrained solution as $w\to\infty$. For finite $w$, it stays well-defined even where the hard feasible set is empty, replacing the strict closed-loop guarantee with one that holds up to a controllable slack vanishing as $w\to\infty$ (Theorem~\ref{thm:soft_safety}). We compare both formulations in our 2D pedestrian experiments.

\subsection{Stage 1: Offline Initialization}
\label{subsec:offline_online_cp}
The key design principle is to separate the field-level \emph{statistical initialization}, computed once offline, from the \emph{lightweight adaptation} performed online.
Offline, we exploit a large calibration set and a flexible GMM to construct the most accurate initial envelope possible, a cost paid only \emph{once}, before deployment.
Online, the GMM is never evaluated; the controller instead adapts the cached envelope directly from streaming residual observations and a prescribed target violation level.

During the offline stage, we assume access to $n_{\mathsf{Cal}}$ independent snapshots of the environment,
\begin{equation*}
\left\{ \mathcal{O}^{(j)}_t : 0 \leq t \leq T,\; 1 \leq j \leq  n_{\mathsf{Cal}}  \right\}
\end{equation*}
from which we obtain the dataset for the $i$-th planning step:
\begin{equation*}
\mathcal{D}_{\mathsf{Cal},i} \coloneqq
\left\{S_{t+i|t}^{(j)}: 0\leq t \leq T - i, 1 \leq j \leq n_{\mathsf{Cal}} \right\}.
\end{equation*}
We then apply FPCA to obtain an orthonormal basis
$
\{\psi_{i,1},\dots,\psi_{i,p_i}\},
$
which produces the coefficient vector of each $S \in \mathcal{D}_{\mathsf{Cal}, i}$ according to~\eqref{eq:coefficient}.

\noindent
{\bf Remark on exchangeability of the calibration set.} The split-conformal coverage~\eqref{eq:split_cp} assumes the calibration and test residual fields are exchangeable. 
In practice the calibration fields $\{S^{(j)}_{t+i\mid t}\}$ are extracted from overlapping prediction windows of a limited number of recordings, so temporally adjacent samples are correlated and exchangeability holds only approximately. 
We mitigate this by subsampling across time when forming $\mathcal{D}_{\mathsf{Cal},i}$ to decorrelate successive fields. Furthermore,  since the formal offline guarantee heavily relies on this assumption, we treat the field coverage of Appendix~\ref{app:field_coverage} as the operative empirical check. 
The online adaptive update, which makes no exchangeability assumption, is a further safeguard against any residual violation.

Using the coefficient dataset
$
\{\xi_i(S_{t+i|t}^{(j)})\}_{j=1}^{n_{\mathsf{Cal}}},
$
we fit a Gaussian mixture model and apply the coefficient-space conformal construction of Section~\ref{sec:gmm_ellipsoids}.
This yields the mixture parameters $\{\widehat\pi_{i,k},\widehat\mu_{i,k},\widehat\Sigma_{i,k}\}_{k=1}^K$, the conformal level $\lambda_i$ (and hence the per-component radii $r_{i,k}\equiv r_{i,k}(\lambda_i)$ of~\eqref{eq:ellipsoid_def}), and the projection residual $\varepsilon_i$.

For each horizon index $i$, we cache the offline tuple
\begin{equation}
\Pi_i
\;\coloneqq\;
\Bigl(
\{\psi_{i,m}\}_{m=1}^{p_i},\,
\{\widehat\pi_{i,k},\widehat\mu_{i,k},\widehat\Sigma_{i,k}\}_{k=1}^K,\,
\lambda_i,\,
\varepsilon_i
\Bigr),
\label{eq:Pi_i_tuple}
\end{equation}
and denote by
\begin{equation}
\boldsymbol{\Pi}\coloneqq\{\Pi_i\}_{i=1}^N
\label{eq:Pi_collection}
\end{equation}
the full collection of horizon-indexed tuples cached by the controller.
Crucially, after deployment the online controller only \emph{evaluates} the closed-form envelope~\eqref{eq:upper-envelope-closed} from these cached parameters and never \emph{re-fits} the GMM. 
In practice, two grids $\widehat\mu_{i,k}^\top\psi_i(\cdot)$ and $(\psi_i(\cdot)^\top\widehat\Sigma_{i,k}\psi_i(\cdot))^{1/2}$ are precomputed once, so an envelope update reduces to recombining them with the (scalar-)adapted radii.

\subsection{Stage 2: Online Adaptation}\label{subsec:online_update}

Once the robot starts moving, the cached envelope is adapted online through a single scalar per horizon index, in the spirit of adaptive conformal prediction~\cite{gibbs2021adaptive}. Because the upper envelope~\eqref{eq:upper-envelope-closed} is controlled by the radii $r_{i,k}$, we introduce an adjustable multiplier $c_{t+i|t}$ on those radii, and define the \emph{parameterized} upper envelope 
\begin{align*}
\mathsf{U}_{t+i|t}(\mathrm{x}; c) \coloneqq \varepsilon_i + \max_{k} \left\{
\widehat\mu_{i,k}^{\top}\psi_i(\mathrm{x})
+
c \cdot r_{i,k}
\sqrt{\psi_i(\mathrm{x})^\top\widehat\Sigma_{i,k}\psi_i(\mathrm{x})}
\right\}
\end{align*}
for $c\ge 0$ (and $\mathsf{U}_{t+i|t}(\mathrm{x}; c) \coloneqq \mathsf{U}_{t+i|t}(\mathrm{x}; 0)$ for $c<0$). 
 The original envelope is recovered at $c=1$: $\mathsf{U}_{t+i|t}(\mathrm{x}) = \mathsf{U}_{t+i|t}(\mathrm{x}; 1)$. Defining the smallest multiplier that covers the realized field,
\begin{equation}\label{eq:multiplier-lower-bound}
a_{t+i|t} \coloneqq
\inf\left\{c\ge 0:S_{t+i|t}(\mathrm{x})\le
\mathsf{U}_{t+i|t}(\mathrm{x}; c)\;\forall \mathrm{x}\in \bar{\mathcal{X}} \right\},
\end{equation}
we take the violation indicator
\begin{equation}\label{eq:online_coeff}
\mathrm{err}^{(t)}_i \coloneqq \mathbb{I}\left[ c_{t+i|t} < a_{t+i|t}  \right],
\end{equation}
which, by~\eqref{eq:multiplier-lower-bound}, is equivalent to
\begin{equation*}
\mathrm{err}^{(t)}_i = \mathbb{I}\!\left[\exists \mathrm{x}\in\bar{\mathcal X}:\ S_{t+i|t}(\mathrm{x}) > \mathsf{U}_{t+i|t}(\mathrm{x}; c_{t+i|t})\right].
\end{equation*}
This keeps the band a valid envelope at all times and yields a clean coverage event. 
At time $t$, once a realized residual field $S_{t+i|t}$ matures for horizon index $i$, we evaluate the indicator~\eqref{eq:online_coeff} on the discretization grid $\bar{\mathcal X}$ alone (cost $O(K p_i |\bar{\mathcal X}|)$, with no full-field reconstruction and no GMM re-fit), and update the multiplier as
\begin{equation}
c_{t+1+i|t+1} =  c_{t+i|t} + \gamma\,\bigl( \mathrm{err}^{(t)}_i - \alpha\bigr).
\label{eq:acp_update}
\end{equation}
Here $c_{t+i|t}$ is the multiplier applied to horizon $i$ at planning time $t$; once the corresponding residual matures, it is updated to $c_{t+1+i|t+1}$ for use at the next planning time. 
The envelope inflates when the realized miscoverage exceeds $\alpha$ and contracts otherwise, driving the long-run miscoverage to $\alpha$. Note that the residual of a prediction made at time \(t\) for horizon
 \(i\) becomes available only at time \(t+i\), after the 
environment state has been observed; the AFCP update therefore uses
delayed feedback, without access to future information.
The corresponding online envelope simply rescales the cached radii,
\begin{align}
 \mathsf{U}^{\text{online}}_{t+i|t}(\mathrm{x})=\varepsilon_i 
 +\max_{k}\left\{\widehat\mu_{i,k}^\top\psi_i(\mathrm{x})+c_{t+i|t}\,r_{i,k}\sqrt{\psi_i(\mathrm{x})^\top\widehat\Sigma_{i,k}\,\psi_i(\mathrm{x})}\right\}.
\label{eq:Ui_online}
\end{align}
Thus, offline GMM fitting only initializes the cached parameters, while online adaptation moves a single scalar $c_{t+i|t}$ (initialized at $c_{i|0}=1$), recomputed against the precomputed grids of~\eqref{eq:Pi_i_tuple} without re-evaluating the GMM.

The following theorem is an immediate consequence of~\cite[Proposition 4.1]{gibbs2021adaptive}.
\begin{theorem}\label{thm:acp-coverage}
The following asymptotic coverage guarantee holds:
\begin{equation*}
\lim_{T\to\infty}\frac{1}{T} \sum_{t=0}^{T-1}\mathbb{I}\left[ S_{t+i|t}(\mathrm{x}) \leq \mathsf{U}^{\text{online}}_{t+i|t}(\mathrm{x}) \; \forall \mathrm{x}\in \bar{\mathcal{X}} \right] = 1 -\alpha.
\end{equation*}
\end{theorem}

\noindent
The exact equality is inherited from~\cite[Proposition 4.1]{gibbs2021adaptive} under its standard
regularity conditions, namely a bounded conformity score, equivalently a bounded
 multiplier $c_{t+i\mid t}$. 
 These hold here because $\bar{\mathcal X}$ is finite and each $S_{t+i|t}$ is bounded on it, so the required multiplier $a_{t+i|t}$ in~\eqref{eq:multiplier-lower-bound} is finite and the recursion~\eqref{eq:acp_update} cannot drift without bound.

The same adaptive-conformal argument applies to the \emph{projection slack}
$\varepsilon_i$, which inflates the envelope uniformly (out of the FPCA subspace)
rather than along the in-subspace directions controlled by $c$. Holding the
in-subspace parameters fixed, write
\begin{equation*}
\mathsf{U}_{t+i|t}(\mathrm{x};\varepsilon)
:= \varepsilon + \max_{k}\Bigl\{\widehat\mu_{i,k}^\top\psi_i(\mathrm{x})
+ r_{i,k}\sqrt{\psi_i(\mathrm{x})^\top\widehat\Sigma_{i,k}\psi_i(\mathrm{x})}\Bigr\},
\end{equation*}
form the field-violation indicator
\[
\mathrm{err}^{\varepsilon,(t)}_i := \mathbb{I}\!\left[\exists\mathrm{x}\in\bar{\mathcal X}:
S_{t+i|t}(\mathrm{x}) > \mathsf{U}_{t+i|t}(\mathrm{x};\varepsilon_{t+i|t})\right],
\]
and update
\begin{equation}
\varepsilon_{t+1+i|t+1}
= \max \bigl\{0,\ \varepsilon_{t+i|t} + \gamma\,(\mathrm{err}^{\varepsilon,(t)}_i-\alpha)\bigr\}.
\label{eq:eps_update}
\end{equation}

\begin{corollary}\label{cor:eps-coverage}
With the in-subspace parameters held fixed and $\varepsilon$ updated by
\eqref{eq:eps_update}, the realized field coverage attains its nominal level:
\begin{equation*}
\lim_{T\to\infty}\frac{1}{T}\sum_{t=0}^{T-1}
\mathbb{I}\!\left[S_{t+i|t}(\mathrm{x})\le \mathsf{U}_{t+i|t}(\mathrm{x};\varepsilon_{t+i|t})
\ \forall\mathrm{x}\in\bar{\mathcal X}\right]=1-\alpha.
\end{equation*}
\end{corollary}
\begin{proof}
Identical to Theorem~\ref{thm:acp-coverage}: \eqref{eq:eps_update} is the adaptive
conformal recursion of~\cite[Prop.~4.1]{gibbs2021adaptive} with the projection slack
$\varepsilon$ as the adapted threshold and the field-violation event as the score,
so the long-run miscoverage converges to $\alpha$ for an arbitrary deployment stream.
\end{proof}

\noindent
This $\varepsilon$-adaptation is the variant we deploy under the deployment-time density shift of Section~\ref{subsec:density-shift}, where the in-subspace fit still transfers but the projection slack must grow. 
Algorithm~\ref{alg:fcp_mpc_acp} summarizes the resulting procedure.

\begin{algorithm}[t]
\caption{Functional CP-informed MPC with offline GMM initialization and optional online adaptive update}
\label{alg:fcp_mpc_acp}
\begin{algorithmic}[1]
\STATE \textbf{Input:} Calibration set $\{\mathcal D_{\mathsf{Cal},i}\}_{i=1}^N$, margin $\delta$, robot radius $r_{\mathrm{robot}}$, discretized workspace $\bar{\mathcal{X}}$, target violation level $\alpha$.
\STATE \textbf{Offline initialization:}
\FOR{each horizon index $i=1,\dots,N$}
    \STATE Compute residual-field samples $S_i^{(j)}(\mathrm{x})=D_{t+i|t}^{(j)}(\mathrm{x})-D_{t+i}^{(j)}(\mathrm{x})$ for all $1\leq j\leq n_{\mathsf{Cal}}$;
    \STATE Apply FPCA and extract basis $\{\psi_{i,m}\}_{m=1}^{p_i}$;
    \STATE Project residual fields to obtain coefficient vectors $\xi_i^{(j)}$;
    \STATE Fit a GMM in coefficient space and conformalize to obtain mixture parameters $\{\widehat\pi_{i,k},\widehat\mu_{i,k},\widehat\Sigma_{i,k}\}$ and radii $\{r_{i,k}\}$;
    \STATE Compute the projection residual $\varepsilon_i$~\eqref{eq:epsilon_quantile};
    \STATE Cache the tuple $\Pi_i$~\eqref{eq:Pi_i_tuple};
\ENDFOR
\STATE Retain only the cached tuples $\{\Pi_i\}_{i=1}^N$;
\STATE \textbf{Online MPC loop} (online adaptation optional):
\FOR{$t=0,1,2,\dots$}
    \STATE Observe current robot state $\mathrm{q}_t$ and obtain obstacle predictions $\{\mathcal O_{t+i|t}\}_{i=1}^N$;
    \FOR{each horizon index $i$ for which an online residual field $S_{t+i|t}$ becomes available}
        \STATE Project $S_{t+i|t}$ onto the cached basis and compute $\xi^{\mathrm{on}}_{i,t}$;
        \STATE Update the scalar radius multiplier $c_{t+1+i|t+1}$~\eqref{eq:acp_update};

    \ENDFOR
    \STATE Sample candidate input sequences and roll out candidate trajectories;
    \STATE Evaluate the envelope $\mathsf{U}^{\text{online}}_{t+i|t}$~\eqref{eq:Ui_online} along rollouts;
    \STATE Filter trajectories violating~\eqref{eq:mpc-constraint-fcp};
    \STATE If no sampled candidate satisfies the hard conformal constraint~\eqref{eq:mpc-constraint-fcp} at $t$, issue a braking-to-hover action;
    \STATE Select the feasible candidate attaining the lowest cost~\eqref{eq:mpc-constraint-fcp};
    \STATE Apply the first control input and advance to the next planning step;
\ENDFOR
\end{algorithmic}
\end{algorithm}

\section{Safety Analysis and Asymptotic Guarantees}
\label{sec:safety}

This section establishes the safety guarantees of the proposed framework.
A key advantage of the functional formulation is that safety holds \emph{at every point of the field}, for \emph{any} planned trajectory satisfying the tightened conformal constraint, independent of how the control space is sampled, because the bound is certified over the entire field rather than along a single path.

We first present the base safety theorem, assuming the functional envelope covers the true score. We then extend it to a robust guarantee for dynamic environments through the online adaptive update, and finally certify the soft-penalty deployment used in dense scenes.

\subsection{Base Safety Guarantee via Functional Coverage}

The following theorem shows that if the spatial envelope asymptotically covers the true score, the closed-loop system is safe.

\begin{theorem}[Asymptotic Safety Guarantee]
\label{thm:asymptotic_safety}
Suppose the FCP-MPC problem~\eqref{eq:mpc-constraint-fcp} is feasible for all $t \ge 0$, and
let $\mathrm{x}_{t}$ be the actual robot position at time $t$.
Suppose that the online envelope $\mathsf{U}^{\text{online}}_{t+1|t}$ satisfies the asymptotic coverage property on $\bar{\mathcal{X}}$ with probability 1:
\[
\lim_{T\to\infty}\frac{1}{T}\sum_{t=0}^{T-1}
\mathbb I\!\left[
S_{t+1|t}(\mathrm{x})\le  \mathsf{U}^{\text{online}}_{t+1|t}(\mathrm{x})\ \ \forall \mathrm{x}\in\bar{\mathcal X}
\right]
=
1-\alpha.
\]
Then the closed-loop system is asymptotically $(1-\alpha)$-safe:
\begin{equation}
\liminf_{T \to \infty}
\frac{1}{T}\sum_{t=0}^{T-1}
\mathbb{I}\!\left[ \mathcal{A}_t \cap \mathcal{O}_t =\varnothing\right]
\ge 1-\alpha,
\qquad \text{w.p. 1.}
\label{eq:asymptotic_safety_guarantee}
\end{equation}
\end{theorem}

\begin{proof}
Since the receding-horizon controller applies only the first planned input at each step, it suffices to certify the applied step $i=1$. For each $t$, define the coverage event
\[
\mathcal E_{t+1}
\coloneqq
\{
S_{t+1|t}(\mathrm{x})\le \mathsf{U}^{\text{online}}_{t+1|t}(\mathrm{x})\ \ \forall \mathrm{x}\in\bar{\mathcal{X}}
\}.
\]
By~\eqref{eq:lower-bound-coverage}, on $\mathcal E_{t+1}$ we have $D_{t+1}(\mathrm{x})\ge \mathsf{L}_{t+1|t}(\mathrm{x})$ for all $\mathrm{x}\in \bar{\mathcal{X}}$.
Since the MPC problem is assumed to be feasible, the applied plan $( \mathrm{x}^*_{t+i|t}: 1\leq i\leq N )$ satisfies the constraints~\eqref{eq:mpc-constraint-fcp}.
Let $\bar{\mathrm{x}}^*_{t+1|t} \coloneqq \argmin_{\bar{\mathrm{x}} \in \bar{\mathcal{X}}} \| \bar{\mathrm{x}} - \mathrm{x}^*_{t+1|t} \|$ be the grid point closest to $\mathrm{x}^*_{t+1|t}$, so that by the definition of $\delta_d$ in~\eqref{eq:resolution},
\[
\|\bar{\mathrm{x}}^*_{t+1|t} - \mathrm{x}^*_{t+1|t}\| \leq \delta_d.
\]
Feasibility gives 
\[
\mathsf{L}_{t+1|t}(\bar{\mathrm{x}}^*_{t+1|t}) \ge r_{\text{robot}}+\delta+\delta_d. 
\]
Since $D_{t+1}$ is $1$-Lipschitz, on $\mathcal{E}_{t+1}$ (where $D_{t+1}\ge\mathsf{L}_{t+1|t}$ on $\bar{\mathcal{X}}$),
\[
D_{t+1}(\mathrm{x}^*_{t+1|t}) \ge D_{t+1}(\bar{\mathrm{x}}^*_{t+1|t}) - \delta_d  \ge r_{\text{robot}}+\delta.
\]
With $\mathrm{x}_{t+1}=\mathrm{x}^*_{t+1|t}$, this gives
\begin{equation}\label{eq:proof_feas_tight}
D_{t+1}(\mathrm{x}_{t+1}) \ge r_{\text{robot}}+\delta,
\end{equation}
hence 
\[
\mathbb I[D_{t+1}(\mathrm{x}_{t+1})\ge r_{\mathrm{robot}} + \delta] \ge \mathbb I[\mathcal E_{t+1}].
\]
Averaging over $t=0,\dots,T-1$ and taking $\liminf$ over $T$, 
\begin{align*}
\liminf_{T\to\infty}&
\frac{1}{T}\sum_{t=0}^{T-1}
\mathbb I\!\left[ \mathcal{A}_t \cap \mathcal{O}_t = \varnothing \right] \\
&\geq\liminf_{T\to\infty}
\frac{1}{T}\sum_{t=0}^{T-1}
\mathbb I\!\left[D_{t+1}(\mathrm{x}_{t+1})\ge r_{\mathrm{robot}} + \delta\right] \\
&\ge
\lim_{T\to\infty}
\frac{1}{T}\sum_{t=0}^{T-1}\mathbb I[\mathcal E_{t+1}] 
=
1-\alpha
\end{align*}
with probability one, which proves~\eqref{eq:asymptotic_safety_guarantee}.
\end{proof}

\noindent
{\bf Remark on discharging the coverage hypothesis (static vs.\ adaptive).}
Theorem~\ref{thm:asymptotic_safety} takes the almost-sure long-run field-coverage limit as a \emph{hypothesis}, and how that hypothesis is met differs between the two deployments. First, with the \emph{static} offline envelope ($c\equiv 1$, no online update), split conformal prediction guarantees the \emph{per-step marginal} coverage~\eqref{eq:coverage-guarantee} under exchangeability. Promoting this marginal statement to the almost-sure time-averaged frequency required by Theorem~\ref{thm:asymptotic_safety} additionally requires the realized coverage events to average out along the deployment stream, e.g., under the (approximate) stationarity that the time-invariance of the residual field (Section~\ref{sec:res-properties}) makes plausible, which we verify empirically as realized long-run field coverage in Appendix~\ref{app:field_coverage} (Tables~\ref{tab:field_coverage}--\ref{tab:field_coverage_horizon}). Second, the \emph{adaptive} (AFCP) deployment discharges the hypothesis \emph{unconditionally}: Theorem~\ref{thm:acp-coverage} delivers the long-run frequency $1-\alpha$ for an arbitrary, even non-exchangeable, stream. This is the sense in which the online update robustifies the static guarantee, and it motivates the robust analysis below.

\subsection{Robust Guarantee Under Online Adaptation}

Theorem~\ref{thm:asymptotic_safety} establishes the foundational logic mapping functional coverage to point-wise safety. In offline CP, the exact coverage equality $\lim_{T\to\infty}\dots=1-\alpha$ relies on exchangeability of the calibration and test data.
In dynamic environments, however, unforeseen distribution shifts can invalidate this assumption.
To robustify the framework, we drop   exchangeability  and employ the AFCP method of Section~\ref{subsec:online_update}.
The following theorem extends the  guarantee to dynamic, non-exchangeable environments.

\begin{theorem}[Asymptotic Safety with AFCP]
\label{thm:asymptotic_safety_afcp}
Consider a dynamic environment where data exchangeability does not hold, and suppose the feasibility assumption of Theorem~\ref{thm:asymptotic_safety} holds. 
If the functional envelope $\mathsf{U}^{\text{online}}_{t+1|t}$ is updated online via the AFCP rule of Section~\ref{subsec:online_update}, then the closed-loop system robustly maintains the asymptotic $(1-\alpha)$-safety:
\begin{equation}
\liminf_{T \to \infty}
\frac{1}{T}\sum_{t=0}^{T-1}
\mathbb{I}\!\left[ \mathcal{A}_t \cap \mathcal{O}_t = \varnothing \right]
\ge 1-\alpha,
\qquad \text{w.p. 1.}
\end{equation}
\end{theorem}

\begin{proof}
Let 
\[
\mathcal E_{t+1} \coloneqq \{ S_{t+1|t}(\mathrm{x})\le \mathsf{U}^{\text{online}}_{t+1|t}(\mathrm{x})\ \forall \mathrm{x}\in\bar{\mathcal{X}} \}.
\]
The per-step implication $\mathbb I[\mathcal{A}_t \cap \mathcal{O}_t = \varnothing] \ge \mathbb I[\mathcal E_{t+1}]$ is established exactly as in Theorem~\ref{thm:asymptotic_safety} [through~\eqref{eq:proof_feas_tight}]; the only change is how the long-run frequency of $\mathcal{E}_{t+1}$ is discharged.
Averaging and taking $\liminf$,
\[
\liminf_{T\to\infty} \frac{1}{T}\sum_{t=0}^{T-1} \mathbb I\!\left[\mathcal{A}_t \cap \mathcal{O}_t = \varnothing\right] \ge \liminf_{T\to\infty} \frac{1}{T}\sum_{t=0}^{T-1}\mathbb I[\mathcal E_{t+1}].
\]
Under the AFCP update, Theorem~\ref{thm:acp-coverage} (at $i=1$) supplies $\lim_{T}\frac1T\sum_t\mathbb{I}[\mathcal{E}_{t+1}]=1-\alpha$ for an arbitrary, even non-exchangeable, stream, without the exchangeability the static case requires. Hence the right-hand side equals $1-\alpha$.
\end{proof}

\subsection{Safety of the Soft Deployment}
\label{subsec:soft_safety}

Theorem~\ref{thm:asymptotic_safety} certifies the \emph{hard} feasibility
filter. 
We now show that the \emph{soft} deployment~\eqref{eq:soft-mpc}, the
variant we recommend in dense scenes where the hard feasible set is
frequently empty, admits the same safety conclusion up to a controllable
slack that vanishes as the penalty weight grows. 
Throughout,
$z=(\mathbf{q},\mathbf{u})$ denotes the (dynamically feasible) decision
variable, $z_w$ a minimizer of the soft program~\eqref{eq:soft-mpc} with
weight $w$, $\mathrm{x}^{w}_{t+1|t}$ the first planned position it applies,
and $\bar{\mathrm{x}}^{w}_{t+1|t}\in\bar{\mathcal{X}}$ its grid projection.

\begin{proposition}[Soft constraint-violation bound]
\label{prop:soft_violation}
Suppose the stage cost is bounded below, $\mathfrak{L}\ge\mathfrak{L}_{\min}$
(e.g., $\mathfrak{L}_{\min}=0$ for nonnegative costs), and that at time $t$
the hard problem~\eqref{eq:mpc-constraint-fcp} is feasible with optimal value
$\mathfrak{L}^{\star}_{t}$. If $z_w$ minimizes the soft
program~\eqref{eq:soft-mpc}, then for every horizon index $1\le i\le N$,
\begin{equation}
\bigl[g_{t+i|t}(\bar{\mathrm{x}}^{w}_{t+i|t})\bigr]_+
\;\le\;
\eta_t(w)\;:=\;\sqrt{\frac{\mathfrak{L}^{\star}_{t}-\mathfrak{L}_{\min}}{w}}.
\label{eq:soft_violation_bound}
\end{equation}
\end{proposition}

\begin{proof}
Let $z^{\star}$ minimize~\eqref{eq:mpc-constraint-fcp}. Being hard-feasible,
it satisfies $g_{t+i|t}(\bar{\mathrm{x}}^{\star}_{t+i|t})\le 0$, so every
hinge $[\,\cdot\,]_+$ vanishes and
$\mathfrak{L}_w(z^{\star})=\mathfrak{L}(z^{\star})=\mathfrak{L}^{\star}_{t}$.
By optimality of $z_w$,
\[
\mathfrak{L}(z_w)+w\sum_{i=1}^{N}\bigl[g_{t+i|t}(\bar{\mathrm{x}}^{w}_{t+i|t})\bigr]_+^2
=\mathfrak{L}_w(z_w)\le \mathfrak{L}_w(z^{\star})=\mathfrak{L}^{\star}_{t}.
\]
Since $\mathfrak{L}(z_w)\ge \mathfrak{L}_{\min}$, we have
\[
\sum_{i=1}^{N}[g_{t+i|t}(\bar{\mathrm{x}}^{w}_{t+i|t})]_+^2
\le (\mathfrak{L}^{\star}_{t}-\mathfrak{L}_{\min})/w.
\]
Each summand is
nonnegative, hence bounded by the sum; taking square roots
gives~\eqref{eq:soft_violation_bound}.
\end{proof}

\noindent
{\bf Remark on the infeasible regime.} The feasibility premise of
Proposition~\ref{prop:soft_violation} is shared with
Theorem~\ref{thm:asymptotic_safety} and is less restrictive than it appears.
Dropping it, the coverage-event chain in the proof of
Theorem~\ref{thm:soft_safety} still gives, for the applied step,
\begin{equation*}
D_{t+1}(\mathrm{x}_{t+1}) \ge r_{\text{robot}} + \delta - v_t,
\qquad v_t := \bigl[g_{t+1|t}(\bar{\mathrm{x}}^{w}_{t+1|t})\bigr]_+,
\end{equation*}
on $\mathcal{E}_{t+1}$, where $v_t$ is the realized conformal violation of the soft
plan. When step $t$ is feasible, Proposition~\ref{prop:soft_violation} bounds
$v_t\le\eta_t(w)$ and Theorem~\ref{thm:soft_safety} is recovered. 
When it is infeasible, no plan meets the conformal constraint, and the penalty term instead drives $v_t$ toward the
minimal achievable violation as $w$ grows; that is, the soft deployment realizes the
least-unsafe plan, and its clearance degrades by exactly the irreducible
infeasibility. The guarantee thus certifies safety wherever it is
certifiable, and elsewhere the soft variant degrades gracefully rather than
aborting.

\begin{theorem}[Asymptotic safety of the soft deployment]
\label{thm:soft_safety}
Let $\bar{\mathfrak{L}}\coloneqq\sup_{t\ge0}\mathfrak{L}^{\star}_{t}<\infty$ (finite since $\mathcal{X}$ is compact and the horizon is finite) and $\eta(w)\coloneqq\sqrt{(\bar{\mathfrak{L}}-\mathfrak{L}_{\min})/w}$. Suppose the hypotheses of Proposition~\ref{prop:soft_violation} hold for all $t\ge0$, and that the online envelope satisfies the asymptotic field-coverage property of Theorem~\ref{thm:asymptotic_safety} (discharged under exchangeability as in Theorem~\ref{thm:asymptotic_safety}, or unconditionally by the AFCP update as in Theorem~\ref{thm:asymptotic_safety_afcp}). Then, the following statements hold:
\begin{enumerate}
\item[(i)] \emph{(Deflated clearance)} With probability one,
\begin{equation*}
\liminf_{T\to\infty}\frac1T\sum_{t=0}^{T-1}
\mathbb{I}\!\left[D_{t+1}(\mathrm{x}_{t+1})\ge r_{\text{robot}}+\delta-\eta(w)\right]
\ge 1-\alpha.
\end{equation*}
\item[(ii)] \emph{(Collision-free)} If the feasibility premise of Proposition~\ref{prop:soft_violation} holds at every step and $w\ge(\bar{\mathfrak{L}}-\mathfrak{L}_{\min})/\delta^{2}$ (equivalently $\eta(w)\le\delta$), then with probability one
\begin{equation*}
\liminf_{T\to\infty}\frac1T\sum_{t=0}^{T-1}
\mathbb{I}\!\left[\mathcal{A}_t\cap\mathcal{O}_t=\varnothing\right]\ge 1-\alpha.
\end{equation*}
\end{enumerate}
Moreover, running the soft penalty~\eqref{eq:soft-mpc} with the inflated margin $\delta+\eta(w)$ in~\eqref{eq:soft_margin} restores $D_{t+1}(\mathrm{x}_{t+1})\ge r_{\text{robot}}+\delta$, and the slack vanishes as $w\to\infty$.
\end{theorem}

\begin{proof}
Work on the coverage event
\[
\mathcal{E}_{t+1}=\{S_{t+1|t}(\mathrm{x})\le\mathsf{U}^{\text{online}}_{t+1|t}(\mathrm{x})\;\forall\mathrm{x}\in\bar{\mathcal{X}}\},
\]
on which $D_{t+1}(\mathrm{x})\ge\mathsf{L}_{t+1|t}(\mathrm{x})$ for all
$\mathrm{x}\in\bar{\mathcal{X}}$ by~\eqref{eq:lower-bound-coverage}.
Proposition~\ref{prop:soft_violation} at $i=1$ and the definition of $g$
in~\eqref{eq:soft_margin} give
\[
\mathsf{L}_{t+1|t}(\bar{\mathrm{x}}^{w}_{t+1|t})\ge
r_{\text{robot}}+\delta+\delta_d-\eta(w).
\]
Repeating the chain of
Theorem~\ref{thm:asymptotic_safety} [cf.~\eqref{eq:proof_feas_tight}] with this
deflated right-hand side yields, with
$\mathrm{x}_{t+1}=\mathrm{x}^{w}_{t+1|t}$,
\[
D_{t+1}(\mathrm{x}_{t+1})\ge
\mathsf{L}_{t+1|t}(\bar{\mathrm{x}}^{w}_{t+1|t})-\delta_d\ge
r_{\text{robot}}+\delta-\eta(w).
\]
Thus
$\mathbb{I}[D_{t+1}(\mathrm{x}_{t+1})\ge r_{\text{robot}}+\delta-\eta(w)]\ge\mathbb{I}[\mathcal{E}_{t+1}]$;
averaging over $t$ and applying the field-coverage limit as in
Theorem~\ref{thm:asymptotic_safety} proves (i). For (ii), $\eta(w)\le\delta$
gives $D_{t+1}(\mathrm{x}_{t+1})\ge r_{\text{robot}}$ on $\mathcal{E}_{t+1}$; the same averaging
concludes. The inflated-margin claim follows by replacing $\delta$ with
$\delta+\eta(w)$ in~\eqref{eq:soft_margin}.
\end{proof}

\noindent
{\bf Remark on empirical obligations and the optimizer premise.}
Theorem~\ref{thm:soft_safety} is deductive and adds \emph{no} new empirical
requirement: its coverage hypothesis is the one already validated for
Theorem~\ref{thm:asymptotic_safety} in Appendix~\ref{app:field_coverage}
(Tables~\ref{tab:field_coverage}--\ref{tab:field_coverage_horizon}), and the
realized closed-loop behavior of the soft deployment is already reported in
Tables~\ref{tab:2d_results}, \ref{tab:pybullet_results},
and~\ref{tab:results_sparse}. The optimizer premise of
Proposition~\ref{prop:soft_violation} is, moreover, met without approximation by
our sampling-based controller: at each step the hard and soft variants score a
common, mode-independent pool of sampled rollouts and apply its argmin under the
respective objective. Since the hard-feasible candidates form a subset of those
scored in soft mode, the comparator $z^{\star}$ lies in the soft minimization
domain and $\mathfrak{L}_w(z_w)\le\mathfrak{L}_w(z^{\star})$ holds exactly. The
sole residual gap is the resolution of the shared sampled pool, identical across
the two deployments.

\section{Experimental Results}\label{sec:res}

\subsection{Planar Navigation with Pedestrians}

We first evaluate FCP-MPC on the ETH--UCY pedestrian benchmark~\cite{pellegrini2009you,lerner2007crowds}, a
standard collection of recorded human-crowd trajectories comprising
five scenes that span a range of crowd
densities. 
In each scene, a robot with a circular footprint of radius
$r_{\mathrm{robot}} = 0.4$\,m navigates toward a fixed
goal while the recorded pedestrians act as the dynamic obstacles
$\mathcal{O}_t$. The robot follows the unicycle dynamics
\begin{equation*}
\begin{bmatrix} x_{t+1} \\ y_{t+1} \\ \theta_{t+1} \end{bmatrix} = \begin{bmatrix}  x_{t} \\ y_{t} \\ \theta_{t} \end{bmatrix} + \Delta t \begin{bmatrix} \cos(\theta_t) & 0 \\ \sin(\theta_t) & 0 \\ 0 & 1 \end{bmatrix} \begin{bmatrix} v_t \\ \omega_t \end{bmatrix},
\end{equation*}
 driven by the sampling-based controller of Section~\ref{sec:cp_mpc}, implemented
with MPPI~\cite{williams2017information} using 1,200 sampled rollouts over a
horizon of $N = 12$ steps at a planning period
$\Delta t = 0.4$\,s. 
All controllers are implemented in Python (NumPy) and run on an Apple~M4 with $16$\,GB of unified memory; reported control times are wall-clock per planning step.

At each planning step the controller consumes multi-step pedestrian
forecasts from Trajectron++~\cite{salzmann2020trajectron}. The residual between
the predicted and realized distance fields is exactly the quantity our functional conformal layer corrects. Following Section~\ref{subsec:online_update}, the envelope is conformalized once offline, per horizon index, from a held-out calibration split ($30\%$ of each scene's residual-field samples) at the target miscoverage level $\alpha = 0.1$, using $K = 7$ GMM components and $p_i = 5$ FPCA modes per step. When enabled, the online AFCP update~\eqref{eq:acp_update} drives the realized miscoverage toward $\alpha = 0.1$. The conformal field is discretized on a $128\times128$ workspace grid, giving a discretization resolution $\delta_d\approx0.18$--$0.27$\,m across scenes~\eqref{eq:resolution}. The resulting offline tuple collection $\{\Pi_i\}_{i=1}^{N}$~\eqref{eq:Pi_collection}
occupies only $\approx5$\,MB, after which the online controller merely consults this cached field.

We deploy the same offline field in two ways, as a hard feasibility filter (\textbf{FCP-MPC~(hard)}) and as a soft penalty (\textbf{FCP-MPC~(soft)}), and compare against three baselines run
under the identical planner: ACP-MPC~\cite{dixit2023adaptive}, CC-MPC~\cite{lekeufack2024decision}, and
ECP-MPC~\cite{shin2025egocentric}. 
To keep the comparison about the conformal \emph{construction} rather than its tuning, every method shares this identical goal-anchored planner and is calibrated to the \emph{same} target miscoverage level $\alpha=0.1$. Each baseline retains its own native conformal parameters (its per-obstacle or per-state radius and adaptive step size) set to that level, so the conservativeness we report for ACP and ECP reflects their score construction rather than an unfavorable hyperparameter choice. 

Each episode terminates when the robot reaches the goal within $0.6$\,m or a step budget of
$100$ steps ($300$ for the larger \texttt{univ} scene) is exhausted (reported as ``timeout''). We report four closed-loop metrics: \textit{(i)} collision rate, \textit{(ii)} infeasible rate,
\textit{(iii)} steps-to-goal, and \textit{(iv)} per-step control time (ms).
A collision is
recorded when the distance from the robot to the nearest pedestrian falls below
$r_{\mathrm{safe}} = r_{\mathrm{robot}} + r_{\mathrm{obs}} \approx 1.11$\,m
(with $r_{\mathrm{obs}} = 1/\sqrt{2}$\,m); 
the soft-penalty methods
carry no infeasible rate by construction (N/A). Collision rate, infeasible rate,
and steps-to-goal are reported as mean\,$\pm$\,standard deviation over $10$ independent
MPPI sampler seeds, each re-randomizing the sampling-based rollouts. 
For each seed we first average over the scene's evaluation windows ($3$ windows per scene), and the seed spread then quantifies the closed-loop
controller's run-to-run variability. Table~\ref{tab:2d_results} reports the
quantitative results and Fig.~\ref{fig:traj_2d} shows representative closed-loop
trajectories.

\subsection{3D Quadrotor Benchmark}\label{sec:quadrotor-nav-benchmark}

\begin{figure}[t!]
    \centering
    \includegraphics[width=0.9\linewidth]{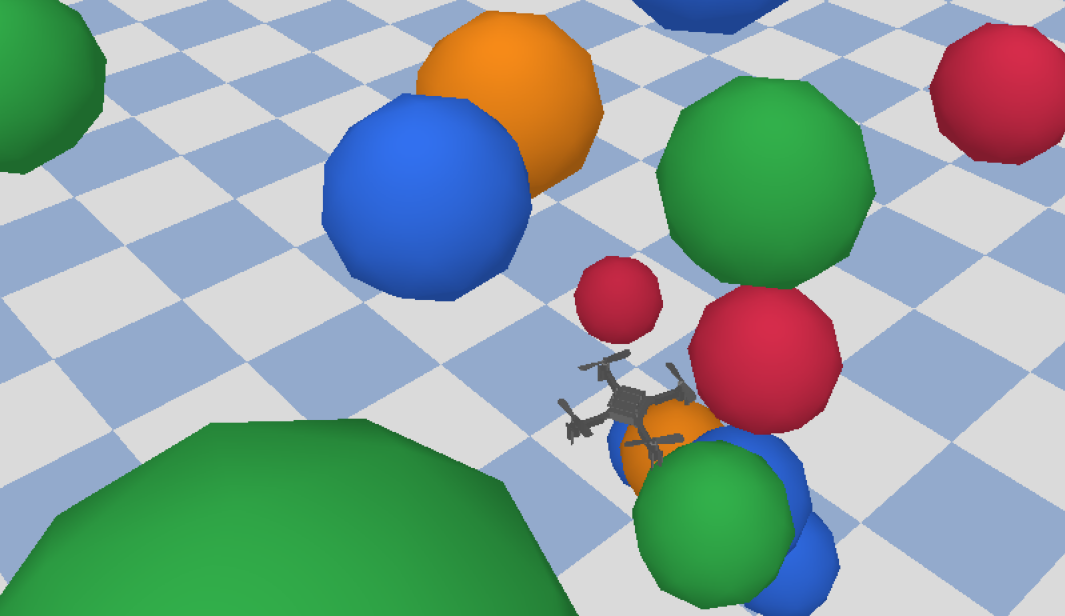}
    \caption{An example snapshot of the quadrotor simulation considered in this paper.}
    \label{fig:quadrotor-example}
\end{figure}

We evaluate FCP-MPC in a physics-based quadrotor navigation task
implemented in PyBullet~\cite{panerati2021learning}. We set $\mathcal{Q} = \mathbb{R}^3$, with $\mathrm{q} \in \mathcal{Q}$ the center of mass (CoM) of the quadrotor, which is velocity-controlled: 
\[
\mathcal{U} = \{ (v_x, v_y, v_z): (v_x^2 + v_y^2)^{1/2} \leq 3,\; |v_z| \leq 0.7  \}.
\]
The quadrotor has radius $r_{\text{robot}}=0.1$ and is equipped with a low-level PID controller running at
240\,Hz, while the high-level MPC planner operates at a slower planning period
$\Delta t = 0.1\,\mathrm{s}$. Since $\Delta t$ is much larger than the PID control period, the quadrotor dynamics  may be  approximated by a first-order integrator,
$
\mathrm{q}_{t+1} = \mathrm{q}_{t} + \Delta t \cdot \mathrm{u}_t.
$

Within a bounded 3D workspace \[
\mathcal{X} = [-3, 7] \times [-3, 7] \times [0, 8],
\]
we spawn $N_{\text{obs}}$ spherical dynamic obstacles of radius $0.2$. 
Each obstacle is a kinematic agent
with stochastic mode switching among several motion patterns, including
constant-velocity, turning, wandering, and stop-and-go behaviors.
Additive process noise 
\[ 
w_t \sim \mathcal{N}(\mathbf{0}, \mathbf{\Sigma}), \quad \mathbf{\Sigma} = \text{diag}(0.22^2, 0.22^2, 0.044^2)
\]
is applied to the obstacle dynamics, inducing increasing
prediction uncertainty over the planning horizon. At each planning step, the controller receives recent obstacle state histories and
predicted obstacle trajectories over a horizon $N = 12$;
a constant-velocity predictor serves as the base prediction model, and prediction
noise is injected to emulate prediction errors.
Oracle future trajectories are used solely for evaluation and metric computation, and the ``oracle future noise'' reported in Section~\ref{sec:quadrotor-nav} perturbs these evaluation-only futures.

The planner samples endpoint candidates and generates smooth quintic polynomial
trajectories in $\mathbb{R}^3$.
Unsafe trajectories are filtered by the conformal lower-bound constraint,
and the minimum-cost feasible one is selected.
Each episode terminates once the robot reaches the goal within
$0.25\,\mathrm{m}$ or a maximum step limit is reached.

Due to the higher-dimensional workspace and the abundance of
free space in 3D, we deploy a \emph{fully offline} functional conformal
envelope here:
unlike the 2D case, excessive conservativeness does not immediately block
progress, so the conformalized worst-case   envelope can be enforced
directly without online adaptation.
The field is discretized on a $40\times40\times40$ grid over the
$10\times10\times8$\,m workspace, giving $\delta_d=0.208$\,m~\eqref{eq:resolution}.
The cached  tuples $\{\Pi_i\}_{i=1}^{N}$ again occupy only a few megabytes,
so online deployment reduces to point-wise queries of this stored field.

We use the same closed-loop metrics as in the 2D study: collision rate,  infeasible
rate,  steps-to-goal, and  per-step control time, with collision now recorded when the minimum
robot--obstacle distance violates $r_{\mathrm{safe}}$. All metrics are averaged over the
evaluated seeds.


\subsection{Properties of the Residual Field}\label{sec:res-properties}

We first examine the two properties of the residual field that justify conformalizing
the envelope offline (Section~\ref{subsec:fcp-why}).
Figures~\ref{fig:fcp_stationarity} and~\ref{fig:fcp_stationarity_scatter} establish
\emph{time-invariance}: per-cell residual statistics estimated from two temporally
disjoint halves of each scene agree closely (cell-mean correlation $r=0.71$--$0.91$
across \texttt{eth}, \texttt{univ}, and \texttt{hotel}), so the field is a stable
property of the scene rather than of a particular time.
Figures~\ref{fig:fcp_lowrank} and~\ref{fig:fcp_scree} establish \emph{low rank}: an FPCA of the residual fields finds that $5$--$7$ components already explain $90\%$ of the
variance on most scenes, with smooth, scene-frame leading eigenfunctions.
Together these confirm that the field is spatially structured and compressible, without reducing to a simple geometric cue such as path curvature or visitation density, which is what lets an offline-conformalized envelope transfer online at low cost.

\begin{figure}[h!]
    \centering
    \includegraphics[width=0.6\linewidth]{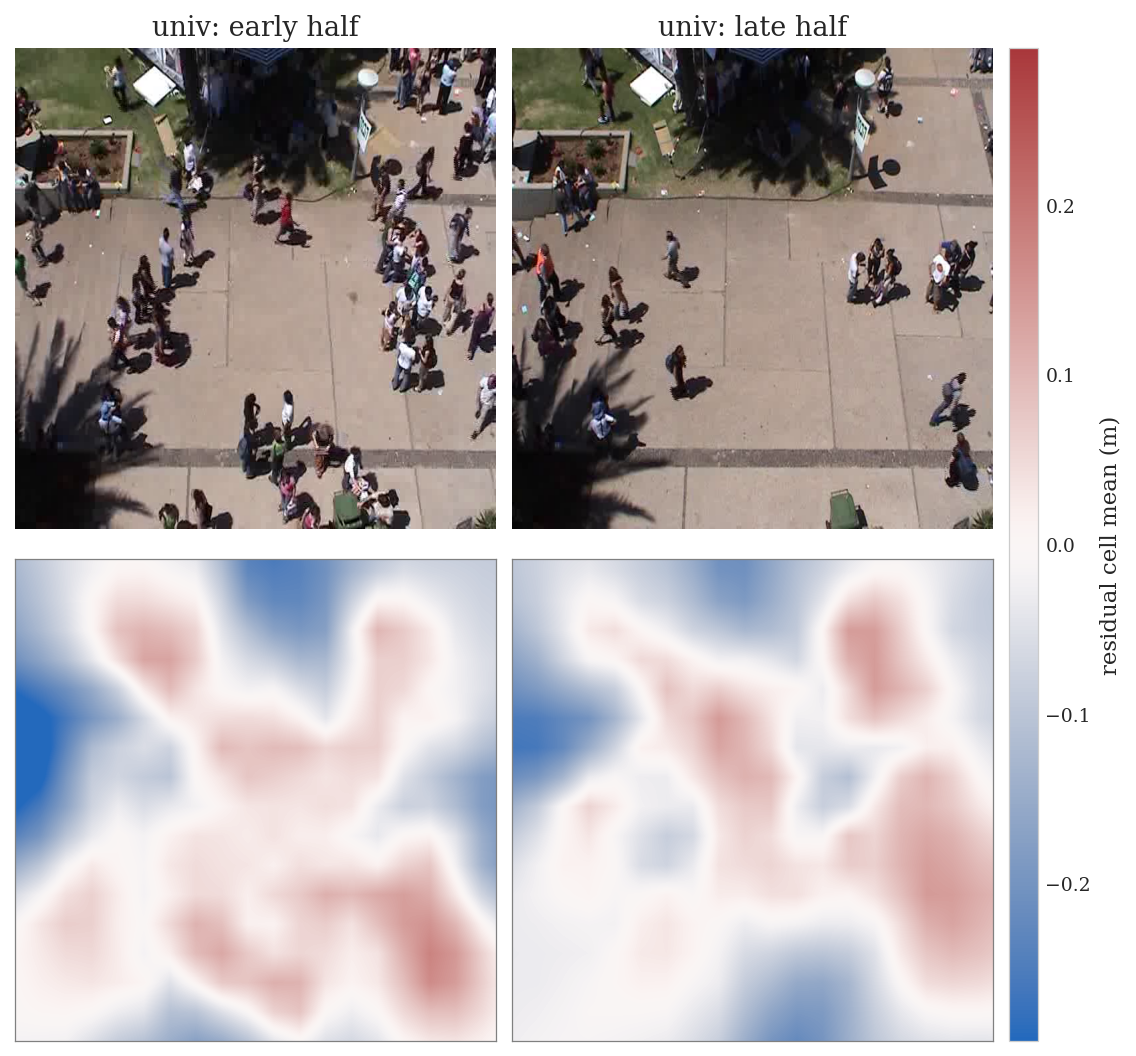}
    \caption{The residual field is approximately \emph{time-invariant}. \textbf{Top:} raw
    camera frames of the \texttt{univ} scene at an early and a late time, cropped to the
    walkable region. \textbf{Bottom:} per-cell means of the residual distance field
    $S_{t+i\mid t}$ estimated from the temporally disjoint early and late halves of the
    recording over the same region (color: mean residual in meters, shared diverging
    scale). These two halves agree closely (Fig.~\ref{fig:fcp_stationarity_scatter}), indicating that the field is a stable property of the scene rather than of a particular time, the premise that lets an offline-conformalized envelope transfer online.}
    \label{fig:fcp_stationarity}
\end{figure}

\begin{figure}[t]
    \centering
    \includegraphics[width=0.5\columnwidth]{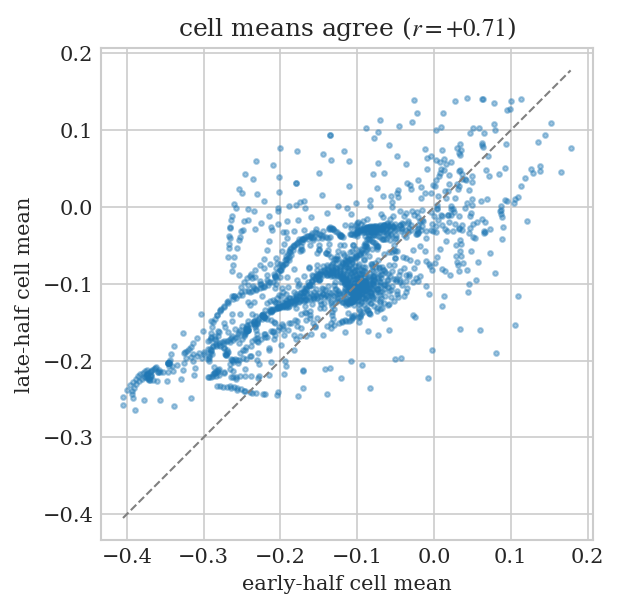}
    \caption{Per-cell residual means of \texttt{univ} from the temporally disjoint early
    and late halves, plotted against each other (one point per spatial cell; the dashed
    line is the identity). They line up closely (Pearson $r=0.71$), quantifying the
    time-invariance visualized in Fig.~\ref{fig:fcp_stationarity}.}
    \label{fig:fcp_stationarity_scatter}
\end{figure}

\begin{figure}[t]
    \centering
    \includegraphics[width=0.99\linewidth]{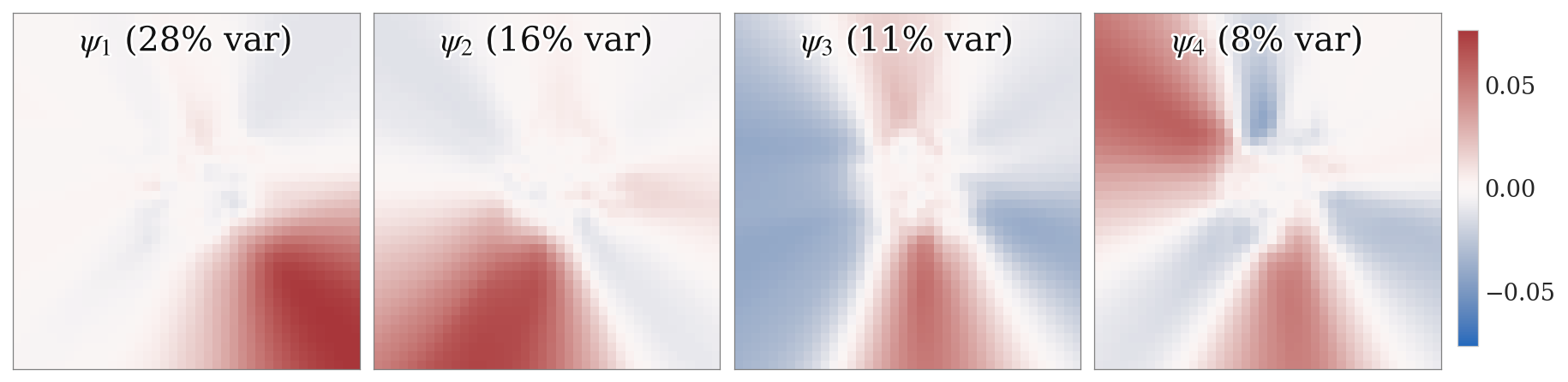}
    \caption{The residual score field is \emph{low-rank}. The four leading functional
    principal components $\psi_1$--$\psi_4$ of the \texttt{univ} residual fields (the
    fraction of variance each explains is shown in parentheses) are smooth, scene-frame
    spatial modes rather than high-frequency noise. A handful of them capture most of the
    field, letting the envelope be learned offline. The cumulative variance
    across all scenes is reported in Fig.~\ref{fig:fcp_scree}.}
    \label{fig:fcp_lowrank}
\end{figure}

\begin{figure}[t]
    \centering
    \includegraphics[width=0.7\columnwidth]{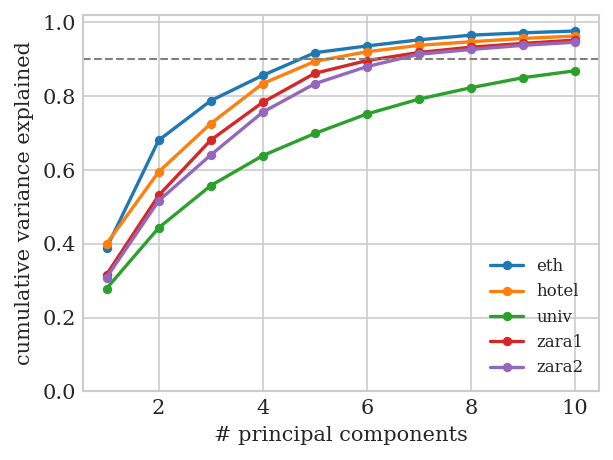}
    \caption{Cumulative variance explained by the leading functional principal components
    on each ETH--UCY scene; $5$--$7$ components already reach $90\%$ (dashed) on most
    scenes (the densest scene, \texttt{univ}, needs more), confirming the residual field
    is low-rank.}
    \label{fig:fcp_scree}
\end{figure}

Figure~\ref{fig:sdf_visualization_2d} visualizes the resulting conformalized signed
distance field on a representative \texttt{zara1} scene. Alongside the ground-truth
field, the conformal lower bound expands the unsafe region around pedestrians whose
predicted motion is more uncertain, widening the safety margin where prediction error
is largest.
Consistent with this, qualitative rollouts show that FCP-MPC adapts its conservativeness to the spatially varying uncertainty, tightening in crowded or uncertain regions while staying permissive in open space.

\begin{figure}[t]
    \centering
    \includegraphics[width=0.95\columnwidth]{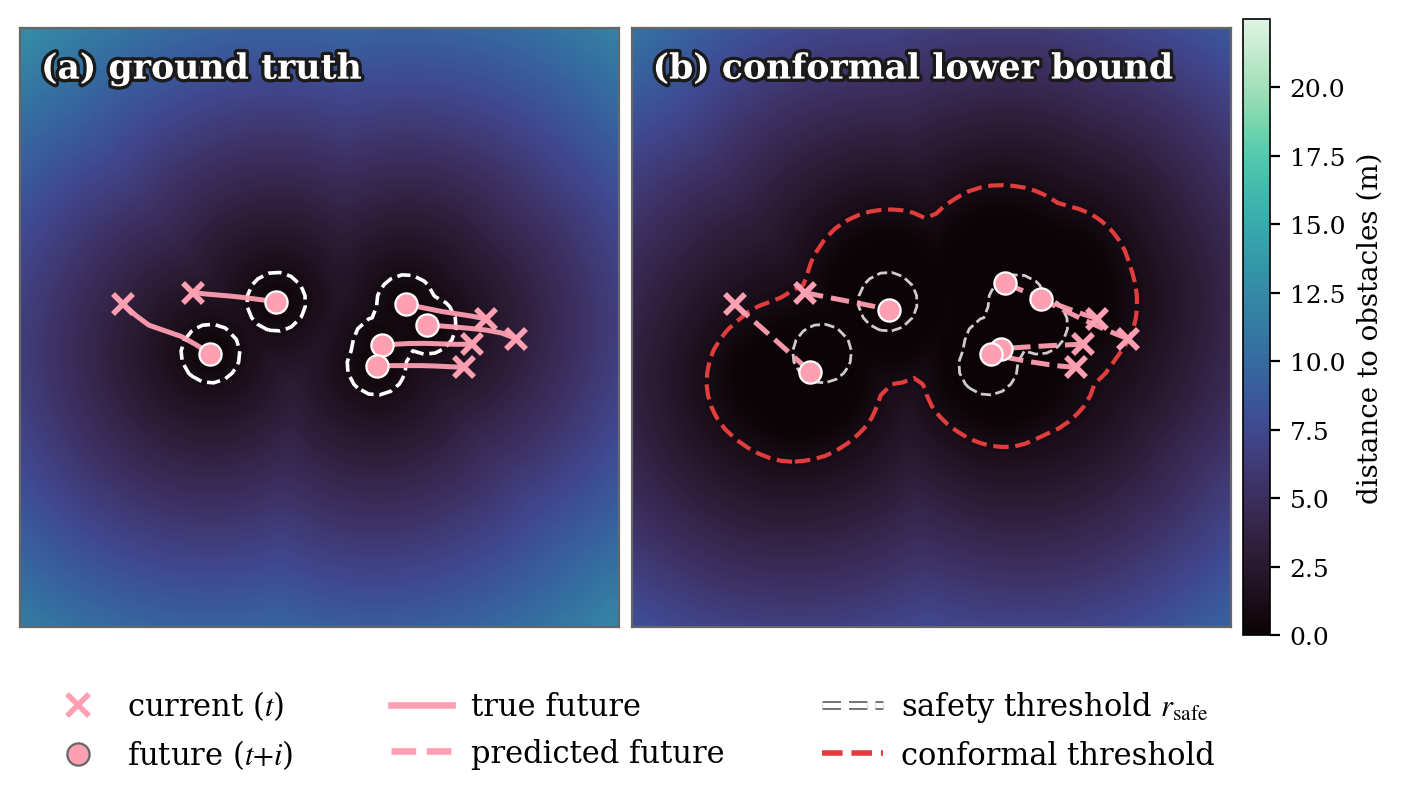}
    \caption{Conformalized distance field on a representative \texttt{zara1} scene,
    looking ahead within the MPC horizon. \textbf{(a)} Ground-truth field, built from the
    pedestrians' \emph{true} future positions; \textbf{(b)} conformal lower-bound field,
    built from their \emph{predicted} future positions and then expanded by the
    offline-calibrated envelope. Color encodes distance to the nearest obstacle (m). Each
    pedestrian is drawn with a cross ($\times$) at its current position and a circle at its
    future position, joined by a solid line for the true future (a) and a dashed line for
    the predicted future (b). White dashed contours mark the nominal safety threshold
    $r_{\text{robot}} + \delta$; the red dashed contour in (b) is the \emph{conformal} threshold,
    which enlarges the unsafe region around pedestrians whose predicted motion is more
    uncertain.}
    \label{fig:sdf_visualization_2d}
\end{figure}

\subsection{2D Pedestrian Navigation}\label{sec:2d-nav}

Table~\ref{tab:2d_results} reports quantitative results on the ETH--UCY pedestrian
datasets, and Fig.~\ref{fig:traj_2d} shows representative closed-loop trajectories
of FCP-MPC and the baselines.
Several consistent trends emerge across all scenes.

The baselines trade off along a single axis. 
CC-MPC has a zero infeasible rate by construction and attains low collision
rates, but is conservative and slow, trading progress for caution. Specifically, it times out in the most crowded scenes
(\texttt{univ}, \texttt{zara2}).
ECP-MPC incurs substantially higher online cost than the other methods
($40$--$83$\,ms versus a few milliseconds for FCP-MPC), though it stays
below the $0.4$\,s planning period in this 2D benchmark.
ACP-MPC is cheap and attains low collisions on some scenes,
but its scalar online adaptation leaves high infeasible rates and frequent timeouts in dense,
dynamic scenes.

FCP-MPC exposes a clear safety--feasibility trade-off between its two deployments of the
same offline field. The \emph{hard} variant enforces the envelope as a feasibility filter and inherits the closed-loop guarantee: collisions stay competitive in the more open scenes
(e.g., \texttt{eth}), but the worst-case envelope compounds over the horizon,
leaving a non-zero infeasible rate and higher collisions in the densest scenes
(\texttt{univ}, \texttt{zara1}). The \emph{soft} variant removes infeasibility on every scene and gives the shortest paths (lowest steps-to-goal on all five scenes), at the cost of a modest collision rate. Both keep the control time far below ECP-MPC (a few milliseconds vs.\ tens), because the uncertainty is quantified offline and the controller only evaluates the cached conformal field online.

\noindent\textbf{Where the hard variant's collisions come from.}
The closed-loop guarantee of Theorem~\ref{thm:asymptotic_safety} applies only on steps where the hard-constrained MPC~\eqref{eq:mpc-constraint-fcp} is feasible, whereas the collision rates in Table~\ref{tab:2d_results} are computed over the \emph{full} episode, including fallback (brake-to-hover) steps taken after infeasibility that lie outside the theorem's premise. Two distinct mechanisms inflate the hard variant's full-episode collisions in the dense scenes, and we separate them rather than attribute everything to fallback. Restricting the collision rate to the steps where the certified filter is actually feasible, it falls within the target $\alpha=0.1$ on \emph{every} scene ($0.000$, $0.006$, $0.026$, $0.027$, $0.044$ on \texttt{eth}, \texttt{hotel}, \texttt{univ}, \texttt{zara1}, \texttt{zara2}; Table~\ref{tab:2d_feasible}), well below the full-episode rates. The excess is carried by the post-infeasibility fallback steps, whose collision rate is $0.33$--$0.37$ on the dense scenes, rather than by failures of the certified filter where it is  active. 

The soft variant tells a complementary story. It collides at a broadly uniform $0.17$--$0.19$ across scenes (\texttt{hotel} lower, $0.045$) rather than concentrating in the harder ones. 
Since the field coverage meets the target throughout (Appendix~\ref{app:field_coverage}), this reflects a progress--safety trade-off of the penalty, not envelope undercoverage (Section~\ref{sec:discussion}).

\begin{table}[!t]
\centering
\small
\caption{2D navigation results on ETH--UCY pedestrian datasets.
Collision rate, infeasible rate, and steps-to-goal are mean\,$\pm$\,std over $10$ MPPI
sampler seeds; lower is better.
The FCP-MPC~(hard) and FCP-MPC~(soft) rows use our core method with the
\emph{offline-conformalized} envelope, so all methods are compared like-for-like;
the optional online adaptation (AFCP)~\eqref{eq:acp_update} is isolated in the
ablation of Table~\ref{tab:2d_ablation}.}
\label{tab:2d_results}
\begin{tabular}{llcccc}
\hline
Dataset & Method & Collision rate $\downarrow$ & Infeasible rate $\downarrow$ & Steps to goal $\downarrow$ & Ctrl.\ time (ms) $\downarrow$ \\
\hline
eth & ACP-MPC~\cite{dixit2023adaptive} & $0.079\pm0.055$ & $0.284\pm0.124$ & $79.0\pm18.7$ & $1.32\pm0.44$ \\
 & CC-MPC~\cite{lekeufack2024decision} & $\mathbf{0.036}\pm0.032$ & N/A & $96.7\pm2.1$ & $1.03\pm0.14$ \\
 & ECP-MPC~\cite{shin2025egocentric} & $0.047\pm0.041$ & $\mathbf{0.239}\pm0.185$ & timeout & $40.88\pm6.20$ \\
 & FCP-MPC (hard) & $0.057\pm0.024$ & $0.700\pm0.111$ & $83.7\pm16.0$ & $\mathbf{0.79}\pm0.04$ \\
 & FCP-MPC (soft) & $0.175\pm0.069$ & N/A & $\mathbf{41.0}\pm4.4$ & $1.80\pm0.20$ \\
\hline
hotel & ACP-MPC~\cite{dixit2023adaptive} & $0.041\pm0.048$ & $0.428\pm0.277$ & $80.3\pm24.4$ & $1.30\pm0.36$ \\
 & CC-MPC~\cite{lekeufack2024decision} & $0.049\pm0.047$ & N/A & $99.3\pm0.6$ & $1.45\pm0.44$ \\
 & ECP-MPC~\cite{shin2025egocentric} & $\mathbf{0.020}\pm0.018$ & $\mathbf{0.213}\pm0.188$ & $85.7\pm24.8$ & $52.70\pm3.79$ \\
 & FCP-MPC (hard) & $0.068\pm0.098$ & $0.357\pm0.243$ & $79.7\pm25.5$ & $\mathbf{1.30}\pm0.24$ \\
 & FCP-MPC (soft) & $0.045\pm0.078$ & N/A & $\mathbf{35.7}\pm14.2$ & $2.03\pm0.15$ \\
\hline
univ & ACP-MPC~\cite{dixit2023adaptive} & $\mathbf{0.001}\pm0.002$ & $\mathbf{0.566}\pm0.151$ & timeout & $4.56\pm0.63$ \\
 & CC-MPC~\cite{lekeufack2024decision} & $0.068\pm0.018$ & N/A & timeout & $5.50\pm0.80$ \\
 & ECP-MPC~\cite{shin2025egocentric} & $0.292\pm0.180$ & $0.664\pm0.370$ & timeout & $82.89\pm3.54$ \\
 & FCP-MPC (hard) & $0.204\pm0.040$ & $0.579\pm0.122$ & $283.3\pm28.9$ & $\mathbf{2.52}\pm0.38$ \\
 & FCP-MPC (soft) & $0.174\pm0.094$ & N/A & $\mathbf{92.3}\pm31.9$ & $6.93\pm0.36$ \\
\hline
zara1 & ACP-MPC~\cite{dixit2023adaptive} & $0.188\pm0.071$ & $0.573\pm0.201$ & timeout & $0.90\pm0.22$ \\
 & CC-MPC~\cite{lekeufack2024decision} & $\mathbf{0.041}\pm0.041$ & N/A & $90.0\pm9.2$ & $1.03\pm0.25$ \\
 & ECP-MPC~\cite{shin2025egocentric} & $0.094\pm0.089$ & $\mathbf{0.161}\pm0.179$ & $82.0\pm31.2$ & $49.30\pm4.69$ \\
 & FCP-MPC (hard) & $0.244\pm0.041$ & $0.679\pm0.171$ & $81.0\pm32.9$ & $\mathbf{0.78}\pm0.15$ \\
 & FCP-MPC (soft) & $0.191\pm0.076$ & N/A & $\mathbf{38.0}\pm2.0$ & $1.80\pm0.18$ \\
\hline
zara2 & ACP-MPC~\cite{dixit2023adaptive} & $0.262\pm0.120$ & $0.628\pm0.389$ & $94.0\pm10.4$ & $1.15\pm0.24$ \\
 & CC-MPC~\cite{lekeufack2024decision} & $\mathbf{0.056}\pm0.016$ & N/A & timeout & $1.27\pm0.34$ \\
 & ECP-MPC~\cite{shin2025egocentric} & $0.176\pm0.085$ & $\mathbf{0.318}\pm0.193$ & timeout & $52.62\pm1.46$ \\
 & FCP-MPC (hard) & $0.181\pm0.078$ & $0.410\pm0.147$ & $92.0\pm13.9$ & $\mathbf{1.11}\pm0.22$ \\
 & FCP-MPC (soft) & $0.189\pm0.150$ & N/A & $\mathbf{65.0}\pm9.5$ & $2.14\pm0.09$ \\
\hline
\end{tabular}
\end{table}

\begin{table}[t]
\centering
\footnotesize
\setlength{\tabcolsep}{4pt}
\caption{Ablation of the online envelope adaptation~\eqref{eq:acp_update} in 2D.
Within each constraint mode (Hard / Soft), the better value per metric is in bold.
Collision, infeasible, and steps-to-goal are mean\,$\pm$\,std over $10$ MPPI sampler seeds;
lower values indicate better performance.}
\label{tab:2d_ablation}
\begin{tabular}{lllcccc}
\hline
Dataset & Constraint & Online adapt. & Collision rate $\downarrow$ & Infeasible rate $\downarrow$ & Steps to goal $\downarrow$ & Ctrl.\ time (ms) $\downarrow$ \\
\hline
eth & Hard & No & $\mathbf{0.057}\pm0.024$ & $\mathbf{0.700}\pm0.111$ & $\mathbf{83.7}\pm16.0$ & $\mathbf{0.79}\pm0.04$ \\
 & Hard & Yes & $0.064\pm0.030$ & $0.766\pm0.181$ & $\mathbf{83.7}\pm16.0$ & $2.23\pm0.15$ \\
 & Soft & No & $0.175\pm0.069$ & N/A & $\mathbf{41.0}\pm4.4$ & $\mathbf{1.80}\pm0.20$ \\
 & Soft & Yes & $\mathbf{0.168}\pm0.063$ & N/A & $43.3\pm4.0$ & $4.57\pm0.27$ \\
\hline
hotel & Hard & No & $0.068\pm0.098$ & $\mathbf{0.357}\pm0.243$ & $\mathbf{79.7}\pm25.5$ & $\mathbf{1.30}\pm0.24$ \\
 & Hard & Yes & $\mathbf{0.040}\pm0.069$ & $0.388\pm0.328$ & $84.3\pm27.1$ & $3.26\pm0.56$ \\
 & Soft & No & $\mathbf{0.045}\pm0.078$ & N/A & $35.7\pm14.2$ & $\mathbf{2.03}\pm0.15$ \\
 & Soft & Yes & $0.081\pm0.140$ & N/A & $\mathbf{30.7}\pm5.5$ & $4.54\pm0.06$ \\
\hline
univ & Hard & No & $0.204\pm0.040$ & $\mathbf{0.579}\pm0.122$ & $\mathbf{283.3}\pm28.9$ & $\mathbf{2.52}\pm0.38$ \\
 & Hard & Yes & $\mathbf{0.178}\pm0.080$ & $0.672\pm0.078$ & timeout & $5.24\pm0.37$ \\
 & Soft & No & $\mathbf{0.174}\pm0.094$ & N/A & $92.3\pm31.9$ & $\mathbf{6.93}\pm0.36$ \\
 & Soft & Yes & $0.207\pm0.092$ & N/A & $\mathbf{85.0}\pm23.4$ & $11.58\pm0.60$ \\
\hline
zara1 & Hard & No & $0.244\pm0.041$ & $\mathbf{0.679}\pm0.171$ & $\mathbf{81.0}\pm32.9$ & $\mathbf{0.78}\pm0.15$ \\
 & Hard & Yes & $\mathbf{0.244}\pm0.068$ & $0.707\pm0.190$ & $\mathbf{81.0}\pm32.9$ & $2.35\pm0.60$ \\
 & Soft & No & $0.191\pm0.076$ & N/A & $\mathbf{38.0}\pm2.0$ & $\mathbf{1.80}\pm0.18$ \\
 & Soft & Yes & $\mathbf{0.190}\pm0.076$ & N/A & $38.3\pm2.5$ & $4.16\pm0.35$ \\
\hline
zara2 & Hard & No & $\mathbf{0.181}\pm0.078$ & $\mathbf{0.410}\pm0.147$ & $\mathbf{92.0}\pm13.9$ & $\mathbf{1.11}\pm0.22$ \\
 & Hard & Yes & $0.200\pm0.066$ & $0.513\pm0.061$ & timeout & $2.72\pm0.24$ \\
 & Soft & No & $\mathbf{0.189}\pm0.150$ & N/A & $\mathbf{65.0}\pm9.5$ & $\mathbf{2.14}\pm0.09$ \\
 & Soft & Yes & $0.208\pm0.134$ & N/A & $66.7\pm11.5$ & $5.07\pm0.20$ \\
\hline
\end{tabular}
\end{table}

\begin{table}[t]
\centering
\small
\caption{Full-episode vs.\ \emph{feasible-step} collision rate for FCP-MPC~(hard) on the
ETH--UCY scenes, measured with per-step feasibility logging. On steps where the certified filter
returns a plan (i.e.\ not a brake-to-hover fallback), the collision rate is within the target $\alpha=0.1$ on every scene; the full-episode
rate is inflated by post-infeasibility fallback steps. Same-run paired values; the full-episode
column reproduces the trend of Table~\ref{tab:2d_results} up to sampler-seed variance.}
\label{tab:2d_feasible}
\begin{tabular}{lcc}
\hline
Scene & Coll. (full-episode) $\downarrow$ & Coll. (feasible-step) $\downarrow$ \\
\hline
\texttt{eth} & 0.057 & \textbf{0.000} \\
\texttt{hotel} & 0.068 & \textbf{0.006} \\
\texttt{univ} & 0.204 & \textbf{0.026} \\
\texttt{zara1} & 0.244 & \textbf{0.027} \\
\texttt{zara2} & 0.181 & \textbf{0.044} \\
\hline
\end{tabular}

\end{table}

\begin{figure}[t]
  \centering
  \includegraphics[width=\linewidth]{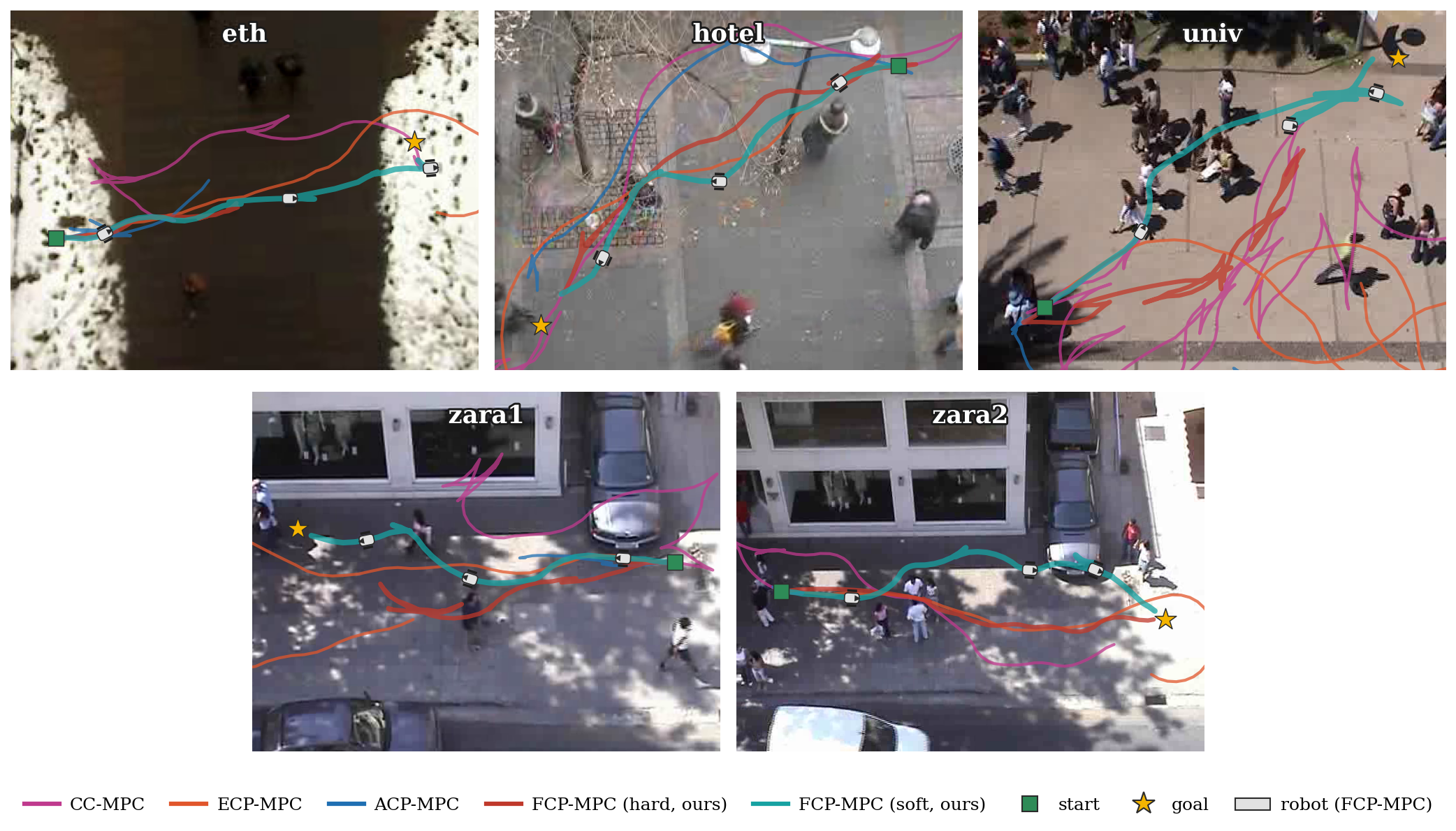}
  \caption{Representative closed-loop trajectories on one scene of each of the five
  ETH--UCY datasets, overlaid on the raw camera frame of each scene (green squares:
  start; gold stars: goal). Both variants of FCP-MPC (ours), \emph{soft} (thick teal) and \emph{hard} (thick red), reach the
  goal along direct paths, whereas the baselines, drawn in lighter lines, tend to wander
  or stall: CC-MPC (magenta) is slow and conservative, ECP-MPC (orange) detours, and
  ACP-MPC (blue) often fails to make steady progress in dense scenes such as
  \texttt{univ}.}
  \label{fig:traj_2d}
\end{figure}

\noindent\textbf{Effect of online envelope adaptation.}
The offline functional envelope is conformalized once on held-out data. Online
adaptation is an \emph{optional} refinement that, in the spirit of adaptive conformal
prediction, rescales the cached radii through a single scalar $c_{t+i|t}$ per horizon
index via~\eqref{eq:acp_update}, driving the realized miscoverage toward
$\alpha$ without re-fitting the GMM.
Table~\ref{tab:2d_ablation} isolates its effect, and the changes are small: the hard-constraint collision and infeasible rates shift only marginally, mostly a slight increase that stays within one standard deviation,
and the soft variant, already feasible by
construction, shifts only slightly. 
This in-distribution neutrality is expected, since the
offline-conformalized envelope already satisfies the required coverage, so the online
update neither helps nor hurts materially when deployment matches calibration.
Because it merely recombines the precomputed grids of Section~\ref{sec:score-comp}, it adds only a few milliseconds per step and stays
below the cost of ECP-MPC.
This sets up the deployment-time density shift of Section~\ref{subsec:density-shift}, where calibration and deployment intentionally differ and the update becomes consequential.

\subsection{3D Quadrotor Navigation}\label{sec:quadrotor-nav}

Recall the 3D setup of Section~\ref{sec:quadrotor-nav-benchmark}: a PyBullet quadrotor threading spherical dynamic obstacles under a constant-velocity predictor, the planner replanning every $\Delta t=0.1$\,s.
Table~\ref{tab:pybullet_results} reports closed-loop results at $N_{\text{obs}}=280$ dynamic
obstacles with noisy predictions (prediction noise $0.2$, obstacle process noise
$0.22$, oracle future noise $0.2$), averaged over $17$ seeds. 
Because absolute timings depend on hardware and implementation, we report per-step control time to compare
how each method's cost \emph{scales} with $N_{\mathrm{obs}}$ on a common platform,
not to assert an absolute real-time threshold.

Since all controllers share the same goal-anchored planner, the differences in
Table~\ref{tab:pybullet_results} stem from the conformal method alone. 
They fall along two
axes: \textit{(i)} how conservative each safety construction is, i.e., whether the robot reaches the
goal at all, and \textit{(ii)} how expensive it is per step.

The two baselines that enforce uncertainty as a \emph{hard} exclusion never reach the goal.
ACP-MPC~\cite{dixit2023adaptive} inflates a scalar conformal radius around each obstacle and
ECP-MPC~\cite{shin2025egocentric} conformalizes an egocentric score at each robot state; both
fly the full horizon, but their inflation is too conservative to thread $280$ moving
obstacles, so they time out on every seed. 
FCP-MPC~(hard) enforces the \emph{same} style of hard feasibility filter
and inherits the guarantee of
Theorems~\ref{thm:asymptotic_safety}--\ref{thm:asymptotic_safety_afcp}, but over a conformal
\emph{field} rather than per-obstacle inflation. 
This is markedly less conservative, yet at
this extreme density even it eventually exhausts the feasible set and
times out. 
No hard filter clears $280$ obstacles, but our field bound is the least conservative of the three and runs at the lowest per-step cost of any method, whereas ECP-MPC is by far the most expensive.

Relaxing the bound from a constraint to a penalty restores progress. 
CC-MPC, a purely soft
formulation with no spatial uncertainty model, reaches the goal on 41\% of seeds but is slow and less safe. FCP-MPC~(soft) applies the
same conformal field as a penalty and is the strongest method overall: it reaches the goal on every seed in the fewest steps, removes infeasibility by construction, and attains the lowest collision rate of any goal-reaching method, at $74.8$\,ms per step on our platform (which is $32\times$ faster than ECP-MPC). Unlike the baselines, its per-step cost
does not grow with $N_{\mathrm{obs}}$.

Reducing the workspace to $N_{\mathrm{obs}}=50$ (Table~\ref{tab:results_sparse}) confirms that
conservativeness, not the planner, drives these outcomes. With ample free space, FCP-MPC~(hard)
now threads the field on $88\%$ of seeds and FCP-MPC~(soft) again reaches
every seed at the lowest collision rate, and CC-MPC also reaches the goal. 
ACP- and
ECP-MPC, however, still time out on every seed, since their per-obstacle and egocentric inflation
stays too tight even when space is plentiful. We return to this density trend in
Section~\ref{sec:discussion}.
Fig.~\ref{fig:fcp_3d} illustrates the conformal bound in 3D, and
Fig.~\ref{fig:traj_3d} compares closed-loop trajectories across three seeds:
FCP-MPC~(soft) reaches the goal along the shortest trajectories,
whereas ACP-MPC and ECP-MPC time out short of it, and
CC-MPC reaches the goal only along substantially longer paths and at a
per-step cost that rises with $N_{\mathrm{obs}}$ (around $100$\,ms at
$N_{\mathrm{obs}}=280$ on our platform).

\begin{figure}[t]
  \centering
  \includegraphics[width=0.6\columnwidth]{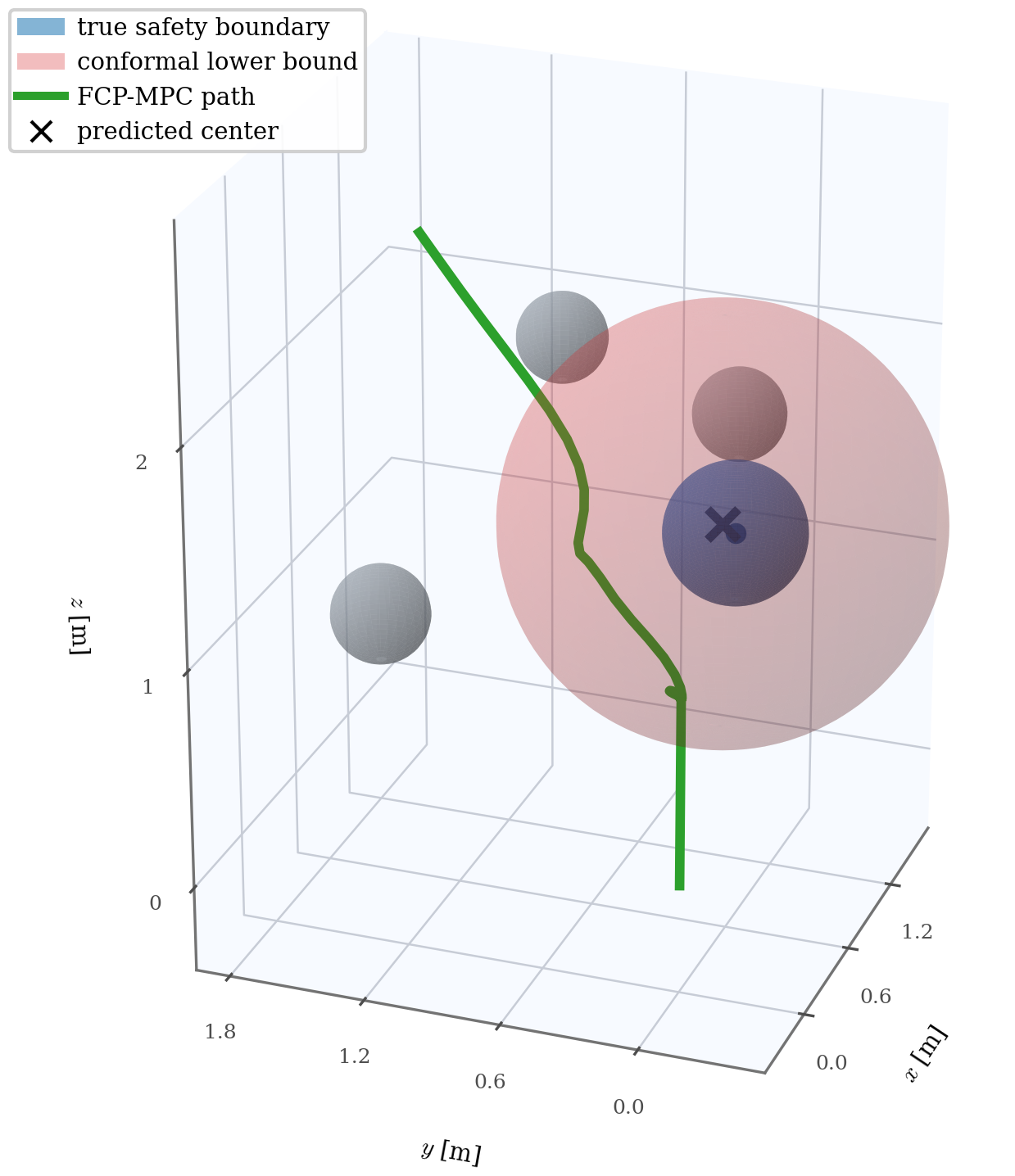}
  \caption{True and conformal safety bounds in 3D, zoomed onto a single obstacle.
  The solid blue ball is the \emph{true} safety boundary
  $D_{t+i}(\mathrm{x})=r_{\text{robot}} + \delta$ around the actual obstacle (used only for
  evaluation); the translucent red shell is the \emph{conformal lower-bound} boundary
  $\mathsf{L}_{t+i\mid t}(\mathrm{x})=r_{\text{robot}} + \delta$ from the functional
  conformal prediction, i.e., the nominal safety region around the predicted center
  ($\times$) inflated by the conformal margin. The conformalized boundary encloses the true boundary here, illustrating the intended field-coverage behavior. The executed FCP-MPC path (green) threads just outside
  it, progressing efficiently while respecting the worst-case safety constraint. Faint
  gray spheres are other obstacles in the scene.}
  \label{fig:fcp_3d}
\end{figure}

\begin{figure}[t]
  \centering
  \includegraphics[width=\linewidth]{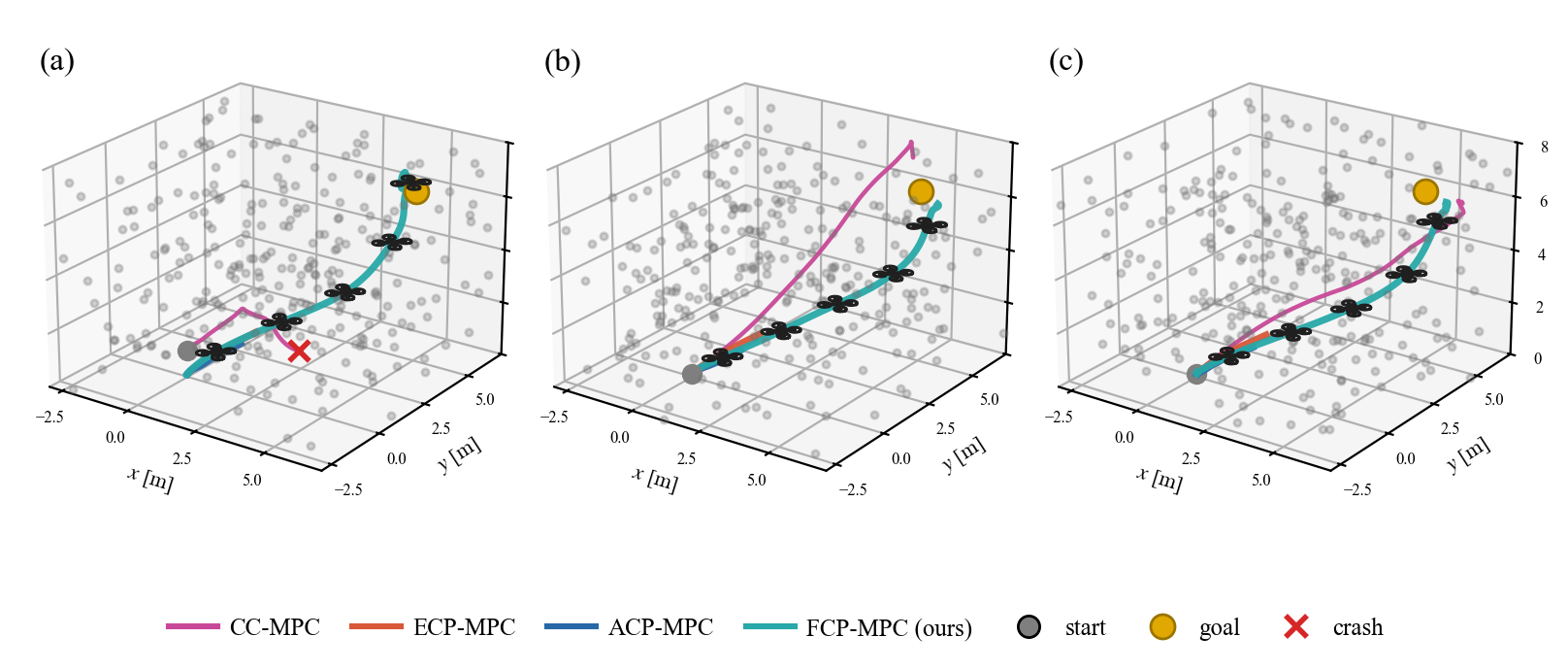}
  \caption{Closed-loop trajectory comparison in the dense 3D quadrotor task
  ($N_{\text{obs}}{=}280$) across three random seeds, shown in
  panels~(a)--(c). All four controllers run under the \emph{same} goal-anchored
  planner: CC-MPC (magenta), ECP-MPC (orange), ACP-MPC (blue), and FCP-MPC (ours, teal).
  The start is a green marker, the goal a gold marker, and a red ``$\times$'' marks where
  a controller crashes. Faint gray dots are the dynamic obstacles. FCP-MPC~(soft) reaches the goal along the shortest paths of any method, whereas ACP-MPC and ECP-MPC are too conservative to thread the dense field and time out short of the goal. CC-MPC also reaches the goal but along substantially longer trajectories and at a per-step cost that grows with $N_{\mathrm{obs}}$ (Table~\ref{tab:pybullet_results}).}
  \label{fig:traj_3d}
\end{figure}

\begin{figure}[t]
  \centering
  \includegraphics[width=0.6\linewidth]{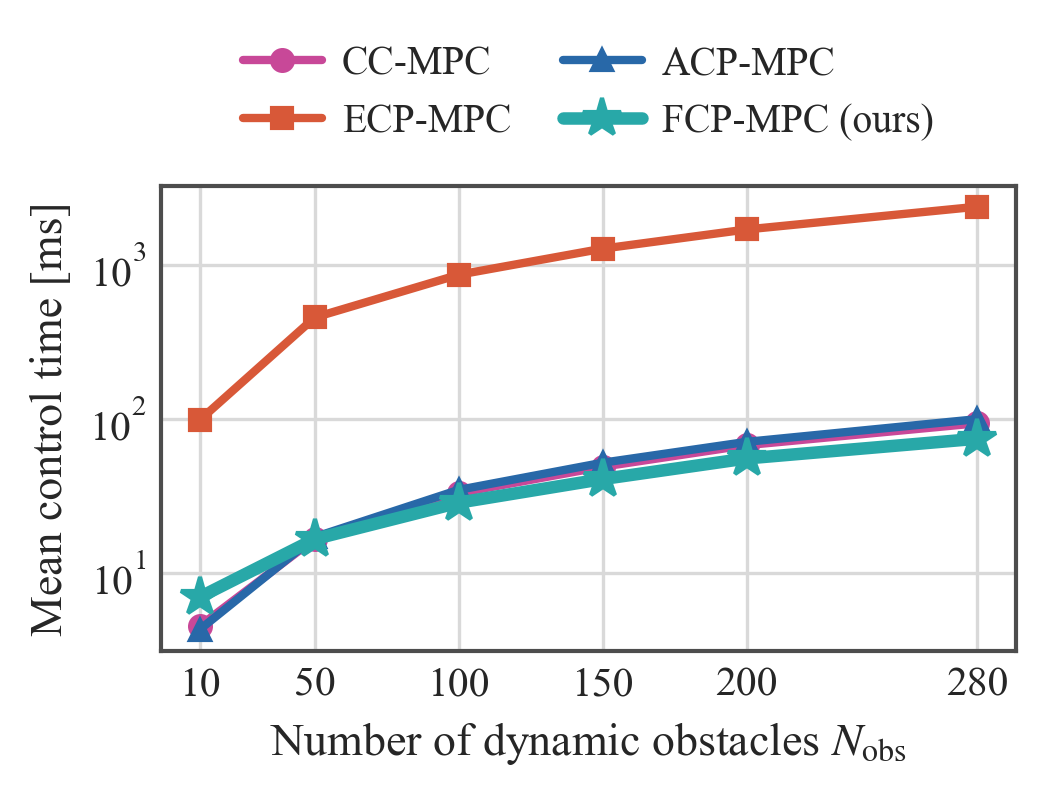}
  \caption{Per-step planning cost versus $N_{\text{obs}}$ in the 3D task: mean control
   time per planning step (log scale, ms) as 
  $N_{\mathrm{obs}}$ grows from $10$ to $280$, for every controller under the shared
  planner. FCP-MPC (ours) reports the soft variant; the hard-filter variant is faster still and likewise flat in $N_{\mathrm{obs}}$.
 ECP-MPC is one to two orders of magnitude slower throughout, and its cost grows with $N_{\mathrm{obs}}$. CC- and ACP-MPC also scale with $N_{\text{obs}}$,
  reaching around $100$\,ms at $N_{\mathrm{obs}}=280$, whereas FCP-MPC (ours) queries a single cached conformal field rather than every obstacle, so its per-step cost stays nearly flat  and remains the lowest of all methods across the range.}
  \label{fig:scalability_3d}
\end{figure}

\begin{table}[t]
\centering
\small
\caption{3D quadrotor navigation results with $280$ dynamic obstacles, reported as
mean $\pm$ std over $17$ seeds ($\alpha=0.1$). FCP-MPC~(hard) and FCP-MPC~(soft) deploy the \emph{same}
offline-conformalized field as a hard feasibility filter and as a soft penalty, respectively.
Steps-to-goal is averaged over successful episodes; ``crashed'' / ``timeout'' mark methods
that never reach the goal. The last column (Goal reached) reports the fraction of the
$17$ seeds on which the controller reached the goal.}
\label{tab:pybullet_results}
\begin{tabular}{lccccc}
\hline
Method &
Collision rate $\downarrow$ &
Infeasible rate $\downarrow$ &
Steps-to-goal $\downarrow$ &
Ctrl.\ time (ms) $\downarrow$ &
Goal reached $\uparrow$ \\
\hline
ACP-MPC~\cite{dixit2023adaptive}
& $0.049\pm0.114$
& $0.125\pm0.161$
& timeout
& 100.3
& 0\% \\
CC-MPC~\cite{lekeufack2024decision}
& $0.053\pm0.040$
& N/A
& $124.4\pm9.5$
& 94.2
& 41\% \\
ECP-MPC~\cite{shin2025egocentric}
& $0.036\pm0.058$
& $\mathbf{0.082}\pm0.068$
& timeout
& 2399.1
& 0\% \\
FCP-MPC (hard)
& $\mathbf{0.023}\pm0.036$
& $0.833\pm0.118$
& timeout
& \textbf{19.9}
& 0\% \\
FCP-MPC (soft)
& $0.028\pm0.032$
& N/A
& $\mathbf{100.4}\pm33.2$
& 74.8
& \textbf{100\%} \\
\hline
\end{tabular}

\end{table}

\begin{table}[!t]
\centering
\small
\caption{3D quadrotor navigation at a lower obstacle count ($N_{\mathrm{obs}}=50$),
reported as mean $\pm$ std over $17$ seeds ($\alpha=0.1$); same setup and columns as
Table~\ref{tab:pybullet_results}. Reporting both densities shows how each method behaves
as the workspace becomes less cluttered.}
\label{tab:results_sparse}
\begin{tabular}{lccccc}
\hline
Method &
Collision rate $\downarrow$ &
Infeasible rate $\downarrow$ &
Steps-to-goal $\downarrow$ &
Ctrl.\ time (ms) $\downarrow$ &
Goal reached $\uparrow$ \\
\hline
ACP-MPC~\cite{dixit2023adaptive}
& $0.005\pm0.010$
& $0.015\pm0.034$
& timeout
& 17.2
& 0\% \\
CC-MPC~\cite{lekeufack2024decision}
& $\mathbf{0.003}\pm0.012$
& N/A
& $162.2\pm34.0$
& 16.8
& 94\% \\
ECP-MPC~\cite{shin2025egocentric}
& $0.005\pm0.011$
& $\mathbf{0.013}\pm0.021$
& timeout
& 458.8
& 0\% \\
FCP-MPC (hard)
& $0.015\pm0.031$
& $0.353\pm0.226$
& $134.1\pm39.4$
& \textbf{8.7}
& 88\% \\
FCP-MPC (soft)
& $0.008\pm0.026$
& N/A
& $\mathbf{78.1}\pm6.9$
& 16.7
& \textbf{100\%} \\
\hline
\end{tabular}

\end{table}

\noindent\textbf{Scalability.}
Figure~\ref{fig:scalability_3d} reports the mean per-step control time as $N_{\mathrm{obs}}$ grows from $10$ to $280$.
Because FCP-MPC checks safety by querying a \emph{cached conformal field} rather than
re-evaluating distances to every obstacle, its per-step cost is largely insensitive to
$N_{\text{obs}}$ and remains the lowest of all methods across the range. The baselines instead scale with the obstacle count: CC- and ACP-MPC rise from tens
of milliseconds to nearly $100$\,ms at $N_{\mathrm{obs}}=280$. ECP-MPC, which also
performs point-wise conformal reasoning across obstacles and horizon steps, is one to two orders of magnitude slower throughout with a cost that also grows in $N_{\mathrm{obs}}$. 
This difference is structural: FCP-MPC's per-step safety check is a single cached-field query, $O(1)$ in the obstacle count, whereas ACP- and CC-MPC are $O(N_{\mathrm{obs}})$ and ECP-MPC is $O(N_{\mathrm{obs}}H)$ per step, independent of hardware and implementation.

\subsection{Field Coverage under Deployment-Time Density Shift}\label{subsec:density-shift}

A central claim of our framework is that the online update sustains the field-coverage guarantee when the deployment stream drifts from the offline calibration data. We test this directly with a \emph{density shift}: the functional envelope is conformalized once, offline, at $N_{\mathrm{obs}}=50$,  then deployed at higher densities ($N_{\mathrm{obs}}\in\{50,80,120\}$) without re-fitting. 
A denser min-of-distances field switches its nearest-obstacle identity more often, so the constant-velocity predictor is systematically wrong at more locations and the realized residual field deviates increasingly from the offline calibration.

Table~\ref{tab:density_shift} reports the realized field coverage at the applied step, measured as the $\forall\mathrm{x}\in\bar{\mathcal{X}}$ event of Theorem~\ref{thm:asymptotic_safety}. 
With a \emph{static} projection
slack $\varepsilon_i$, the offline envelope fails to cover the full-field event and degrades as density grows ($0.080\!\to\!0.034$): the fixed slack budgets only for the projection residual seen offline, whereas the deviations realized at test time exceed it. Adapting $\varepsilon_i$ online by the projection-slack recursion~\eqref{eq:eps_update}
(Corollary~\ref{cor:eps-coverage}) restores the target $1-\alpha=0.9$
coverage and \emph{holds it across all densities}. 

Notably, collisions remain small across all densities and both settings, including the static configuration whose coverage collapses. There, receding-horizon replanning corrects each transient envelope violation at the next observation, before it becomes a collision. Safety and field coverage thus come apart in this regime: coverage is sufficient for safety but not, here, necessary. We take up what the coverage certificate buys, once the feedback structure already absorbs most of the coverage loss, in Section~\ref{sec:discussion}.

\begin{table}[t]
\centering
\caption{Realized field coverage under a deployment-time \emph{density shift}. The functional
envelope is conformalized once offline at $N_{\mathrm{obs}}=50$ and deployed at higher obstacle
densities without re-fitting. ``Static'' keeps the offline projection slack fixed; ``Adaptive''
updates it online. Field coverage is the realized $\forall\mathrm{x}\in\bar{\mathcal{X}}$ event at
the applied step ($1-\alpha=0.9$ target), while ``near-obs.'' restricts it to cells within $1$\,m of an
obstacle. Mean over $10$ seeds; collisions remain small ($\leq 0.05$).}
\label{tab:density_shift}
\begin{tabular}{llccc}
\hline
$N_{\mathrm{obs}}$ & Method & Field cov.\ ($\forall\mathrm{x}$) & Near-obs.\ cov. & Collision \\
\hline
$50$   & Static   & $0.080\pm0.011$ & $0.091\pm0.010$ & $0.036\pm0.092$ \\
               & Adaptive & $\mathbf{0.897}\pm0.021$ & $\mathbf{0.898}\pm0.021$ & $0.032\pm0.083$ \\
\hline
$80$           & Static   & $0.051\pm0.012$ & $0.061\pm0.013$ & $0.005\pm0.013$ \\
               & Adaptive & $\mathbf{0.891}\pm0.022$ & $\mathbf{0.893}\pm0.021$ & $0.004\pm0.012$ \\
\hline
$120$          & Static   & $0.034\pm0.008$ & $0.040\pm0.008$ & $0.035\pm0.103$ \\
               & Adaptive & $\mathbf{0.909}\pm0.020$ & $\mathbf{0.910}\pm0.019$ & $0.049\pm0.109$ \\
\hline
\end{tabular}
\end{table}

\section{Discussion}
\label{sec:discussion}

\noindent\textbf{Enforcing the safety bound in 2D.}
The 2D benchmarks isolate how the conformal lower bound should be \emph{used}
once it has been constructed.
The hard variant enforces the bound as a feasibility filter and thus carries the
closed-loop guarantee. 
As the worst-case envelope compounds over the horizon, however,
it can become overconservative, leaving a non-zero infeasible rate and higher collisions
in the harder scenes. 
The soft variant resolves this  by treating the bound as a penalty rather than a hard constraint: it yields the shortest paths, trading a small, density-independent collision cost for that progress (Table~\ref{tab:2d_results}).
The optional online adaptation of~\eqref{eq:acp_update} is a low-cost refinement whose
effect on these in-distribution benchmarks is marginal and scene-dependent
(Table~\ref{tab:2d_ablation}): the offline-conformalized envelope does the heavy lifting
when deployment matches calibration. 
Its value appears instead under distribution shift, which we examine separately in Section~\ref{subsec:density-shift}.
Both variants reuse the offline conformal prediction, so any online cost is only a few milliseconds, far below the tens required by ECP-MPC's explicit
online uncertainty reasoning. These pedestrian experiments thus exercise the fieldwise conformal prediction where free space is scarce, complementing the open 3D study discussed next.

\noindent\textbf{Separating conformal prediction from real-time planning in 3D.}
The two ways the baselines fail, conservativeness and computation cost (Section~\ref{sec:quadrotor-nav}), share the same cause: they quantify uncertainty \emph{online}, per obstacle or per robot state. 
This is both costly at runtime and forces a locally tight margin around every object. FCP instead
estimates the conformal \emph{field} once offline. At runtime, a single cached-field query suffices, and the result is far less conservative because the bound is reasoned globally over space rather than inflated object by object. This is why
FCP is at once the cheapest method and the only conformal construction whose hard filter 
clears the field when free space is adequate.

\noindent\textbf{Certified hard filter vs.\ practical soft deployment.}
The same offline field can be deployed in two ways (Section~\ref{subsec:mpc_formulation}). As a \emph{hard} filter it
carries the formal coverage guarantee of
Theorems~\ref{thm:asymptotic_safety}--\ref{thm:asymptotic_safety_afcp}, at the
price of conservativeness that grows with density: the infeasible rate
rises from $0.353$ at $N_{\mathrm{obs}}=50$, where the filter still reaches the
goal on $88\%$ of seeds at the lowest per-step cost of any method, to $0.833$ at
$N_{\mathrm{obs}}=280$, where the worst-case envelope leaves so few feasible
candidates that the planner stalls and times out. 
As a \emph{soft} penalty the field carries a slightly weaker but still formal guarantee (Theorem~\ref{thm:soft_safety}): on feasible steps it is safe up to a slack that
vanishes as $w\to\infty$, and elsewhere it degrades gracefully to the
least-violating plan. 

Which branch of the soft guarantee applies depends on the density. The \emph{exact} branch presumes a non-empty hard feasible set at every step, and this premise is itself empirically violated at $N_{\mathrm{obs}}=280$ (hard infeasible rate $0.833$).
The assurance we actually rely on there is therefore the
\emph{degradation} branch of Theorem~\ref{thm:soft_safety} together with the
empirical per-step coverage of Section~\ref{subsec:density-shift}, while the clean $1-\alpha$ statement
governs the moderate-density regime in which the hard filter remains feasible.
Empirically, the soft deployment removes infeasibility entirely and attains
the lowest collision rate among goal-reaching methods ($0.028$ at $N_{\text{obs}}=280$, within
$\alpha=0.1$), reaching the goal on every seed at both densities. Notably,
FCP-MPC stays within the target miscoverage level \emph{while actively traversing}
the dense field, whereas ACP's low collision count reflects a controller that
never threads it, and ECP is both slow and
over-conservative.  We attribute this reliable
goal-reaching to the spatial coherence of the field-level bound, which exposes globally safe corridors rather than forcing locally conservative, point-wise
reasoning.

\noindent\textbf{On the value of coverage in closed-loop scenarios.}
Section~\ref{subsec:density-shift} shows that closed-loop replanning absorbs \emph{most} of a coverage loss: even when the static envelope's field coverage collapses, the realized collision rate stays low because receding-horizon correction catches transient violations at the next observation. 
It does not, however, catch \emph{all} of them, as the static configuration's small but nonzero collisions in Table~\ref{tab:density_shift} indicate. It is therefore worth stating precisely what the field-coverage guarantee and the AFCP update provide that feedback alone does not. 
First, the coverage guarantee is \emph{planner-agnostic}: every feasible candidate at a step is safe, whereas a low closed-loop collision rate only certifies that the executed trajectories happened to avoid contact, with no assurance under a different sampler, seed, or planner. 
Second, it is \emph{pre-hoc}: it prevents violations from arising in the first place, whereas closed-loop correction is reactive and presumes that the next observation arrives, and a feasible escape exists, before a transient violation becomes a collision. Our fast, open, and slow-obstacle settings largely satisfy these recovery assumptions, which is why feedback suffices here; under high travel speed, tight corridors, or sensing latency they may fail, and the pre-hoc guarantee is what remains. The residual collisions already visible in Table~\ref{tab:density_shift} are a mild instance of this gap; a dedicated study of the non-recoverable regime is left to future work.

This fixes the role of each deployment. 
The hard filter is the \emph{exactly} certified variant, appropriate when free space is adequate; the soft penalty is the deployment we recommend in the densest scenes, where strict enforcement is too conservative to make progress; and the online update may be engaged where the formal guarantee is wanted, or disengaged where closed-loop correction suffices and feasibility is at a premium.
The hard/soft and AFCP-on/off choices thus fall on one axis: the framework selects, by deployment context, how much of its safety to take as a proactive certificate versus closed-loop correction, rather than committing to one everywhere.

\section{Conclusion}

We presented a functional conformal prediction framework that conformalizes a
high-confidence lower bound of the predicted distance field and enforces it as a
tightened safety constraint within a sampling-based model predictive controller,
FCP-MPC. By conformalizing uncertainty at the level of the spatial \emph{field} rather
than at individual states or trajectories, the method reasons about safety globally: the resulting certificate holds for any trajectory
satisfying it, independent of the control sampler. Because the residual field is low-rank and approximately time-invariant, its offline--online decomposition keeps per-step computation low, with a GMM-based envelope fitted offline and 
 only a single scalar per horizon index refined online by a lightweight
adaptive update. We established asymptotic closed-loop safety both under exchangeability and, through
the online adaptation, under distribution shift.
Across the ETH--UCY pedestrian benchmarks and a dense 3D quadrotor task with up to $280$ moving obstacles, FCP-MPC attained a favorable balance of safety,
feasibility, and efficiency. Its  per-step cost stayed largely insensitive to obstacle count, scaling far more gracefully than online uncertainty-reasoning baselines.

\noindent
{\bf Limitations and future work.}
Because the conformalized envelope bounds the worst case, the hard-constrained variant
remains conservative and can still report a non-zero infeasible rate in the densest
2D scenes; the soft and online-adapted variants mitigate, but do not eliminate, this
trade-off.
Our guarantees further assume a \emph{fixed} functional basis. Relaxing these assumptions and learning the basis online, folding perception and state-estimation error into the same conformal pipeline, and extending the method to richer robot and obstacle dynamics are all promising directions.

\appendix

\section{Derivation of Equation~\eqref{eq:upper-envelope-closed}}\label{app:derivation}
We derive the closed-form expression of $\breve{\mathsf{U}}_{t+i|t}$.
Since the conformity score $g$ is a maximum over mixture components, the prediction set $\mathscr{T}_{t+i\mid t}$ decomposes exactly as a union of component-wise sublevel sets 
(cf.~\cite[Eq. (9)]{lei2015conformal}):
\begin{equation}
\begin{aligned}
\mathscr{T}_{t+i\mid t} &= \bigcup_{k=1}^K \mathscr{T}^{(k)}_{t+i\mid t}, \\
\mathscr{T}^{(k)}_{t+i\mid t} &\coloneqq \left\{\xi:\ \mathcal{N}(\xi;\widehat\mu_k,\widehat\Sigma_k)\ge \tfrac{\lambda_i}{\widehat\pi_k}\right\}.
\end{aligned}
\label{eq:exact_union}
\end{equation}
Each $\mathscr{T}^{(k)}_{t+i\mid t}$ is an ellipsoid
\begin{equation}
\mathscr{T}^{(k)}_{t+i\mid t}
=
\left\{\xi:\ (\xi-\widehat\mu_k)^\top \widehat\Sigma_k^{-1}(\xi-\widehat\mu_k)\le r_{k,i}^2\right\},
\label{eq:ellipsoid_def}
\end{equation}
where the radius is
\begin{equation*}
r_{k,i}^2
=
\max\!\left\{
0,\,
-2\log\!\left(
\frac{\lambda_i}{\widehat\pi_k}\,(2\pi)^{p_i/2}\sqrt{\det(\widehat\Sigma_k)}
\right)
\right\}.
\end{equation*}
The $\max\{0,\cdot\}$ covers the degenerate case in which the threshold $\lambda_i/\widehat\pi_k$ exceeds the component's peak density, so that the sublevel set is empty.

Recall from Section~\ref{sec:gmm_ellipsoids} that $\psi_i(\mathrm{x})\in\mathbb{R}^{p_i}$ collects the basis evaluations at $\mathrm{x}$, so the projected value of a coefficient vector $\xi$ is $\xi^\top\psi_i(\mathrm{x})$. 
Using the union~\eqref{eq:exact_union},
\begin{equation*}
\breve{\mathsf{U}}_{t+i|t}(\mathrm{x}) \coloneqq \sup_{\xi\in\mathscr{T}_{t+i\mid t}} \xi^\top \psi_i(\mathrm{x})
=
\max_{1\le k\le K}\ \sup_{\xi\in\mathscr{T}^{(k)}_{t+i\mid t}} \xi^\top \psi_i(\mathrm{x}).
\end{equation*}
For an ellipsoid~\eqref{eq:ellipsoid_def}, the inner supremum is given as:
\begin{equation}
\sup_{\xi\in\mathscr{T}^{(k)}_{t+i\mid t}} \xi^\top \psi_i(\mathrm{x})
=
\widehat\mu_k^\top \psi_i(\mathrm{x})
+
r_{k,i}\sqrt{\psi_i(\mathrm{x})^\top \widehat\Sigma_k\, \psi_i(\mathrm{x})}.
\label{eq:ellipsoid_support_function}
\end{equation}
Taking the maximum over $k$ yields the desired representation~\eqref{eq:upper-envelope-closed}.
For numerical stability, we add a small diagonal jitter to $\widehat\Sigma_k$.

\begin{table}[t]
\centering
\caption{Empirical field coverage at the applied horizon $i=1$ on held-out
ETH--UCY scenes: the fraction of test scenes on which
$D_{t+1}(\mathrm{x})\ge \mathsf{L}_{t+1\mid t}(\mathrm{x})$ for all $\mathrm{x}\in\bar{\mathcal{X}}$,
as the target $1-\alpha$ is swept. The pooled \textbf{All} row meets the target at every level.}
\label{tab:field_coverage}
\begin{tabular}{lcccc}
\hline
Scene & \multicolumn{4}{c}{Target field coverage $1-\alpha$} \\
\cline{2-5}
 & $0.95$ & $0.90$ & $0.80$ & $0.70$ \\
\hline
\texttt{eth} & 99.4 & 99.4 & 96.4 & 95.2 \\
\texttt{hotel} & 99.6 & 98.2 & 94.6 & 89.7 \\
\texttt{univ} & 95.4 & 92.6 & 83.3 & 67.6 \\
\texttt{zara1} & 96.5 & 91.3 & 85.5 & 80.3 \\
\texttt{zara2} & 96.2 & 93.3 & 87.6 & 80.9 \\
\hline
\textbf{All} & 97.6 & 95.2 & 90.1 & 84.1 \\
\hline
\end{tabular}
\end{table}

\begin{table}[t]
\centering
\caption{Per-horizon empirical field coverage at the deployed level $\alpha=0.1$:
the same event resolved by prediction-horizon index $i$, with the envelope
conformalized once offline per horizon index. Coverage holds at or above the target
$1-\alpha=0.9$ across the full $N=12$ horizon.}
\label{tab:field_coverage_horizon}
\begin{tabular}{lccccc}
\hline
Scene & \multicolumn{4}{c}{Horizon index $i$ (target $1-\alpha=0.90$)} & Pooled \\
\cline{2-5}
 & $i=1$ & $i=4$ & $i=8$ & $i=12$ & $i\!=\!1..12$ \\
\hline
\texttt{eth} & 99.4 & 98.7 & 100.0 & 100.0 & 98.6 \\
\texttt{hotel} & 98.2 & 99.0 & 98.9 & 98.7 & 98.7 \\
\texttt{univ} & 92.6 & 88.9 & 94.4 & 95.3 & 92.9 \\
\texttt{zara1} & 91.3 & 94.7 & 96.3 & 96.9 & 95.4 \\
\texttt{zara2} & 93.3 & 92.8 & 95.6 & 98.0 & 95.0 \\
\hline
\end{tabular}

\end{table}

\section{Empirical Field-Coverage Validation}\label{app:field_coverage}

We corroborate that the conformalized distance field attains its target level by measuring
the empirical \emph{field coverage}: the frequency of the functional coverage event
$\{ D_{t+i}(\mathrm{x})\ge \mathsf{L}_{t+i\mid t}(\mathrm{x})\ \ \forall \mathrm{x}\in\bar{\mathcal{X}}\,\}$
that underlies our coverage guarantee stated in Theorems~\ref{thm:asymptotic_safety}--\ref{thm:asymptotic_safety_afcp} (established via Theorem~\ref{thm:acp-coverage}). 
On a held-out $20\%$ split of each ETH--UCY
scene's Trajectron++ prediction samples, we mark a scene \emph{uncovered} if the conformalized lower-bound field is violated at even a single grid point $\mathrm{x}\in\bar{\mathcal{X}}$, and average this indicator over the test scenes. 
The envelope is fit on the complementary split exactly as deployed:
 per horizon index, $K=7$ GMM components, $p_i=5$ FPCA modes, and a $30\%$ calibration
fraction for the conformal quantile.

Table~\ref{tab:field_coverage}  reports the empirical coverage at the applied step $i=1$ (the step to which the closed-loop guarantee applies) as the target $1-\alpha$ is swept. The pooled \textbf{All} row meets the target at every level, supporting the coverage guarantee.
The only \emph{sub-target} entry, \texttt{univ} at the loosest target $1-\alpha=0.7$,
reflects finite-sample fluctuation on the densest scene, where the margin to the target is smallest and a single held-out window can tip the pooled fraction.

Table~\ref{tab:field_coverage_horizon} reports the same event at the deployed level $\alpha=0.1$, resolved by prediction-horizon index $i$. Coverage holds at or above $1-\alpha=0.9$ over the full $N=12$ horizon, because the
per-horizon envelope widens with $i$ to absorb the degrading predictor. By
construction the $i=1$ column of Table~\ref{tab:field_coverage_horizon} equals the
$1-\alpha=0.9$ column of Table~\ref{tab:field_coverage}.

\bibliographystyle{IEEEtran}
\bibliography{reference}

\end{document}